\documentclass[11pt]{article}
\usepackage{hyphenat}
\usepackage{sibarticle}
\usepackage{lmodern}
\usepackage{url}
\usepackage{nicefrac}
\usepackage{microtype}
\usepackage{amssymb,graphicx}
\usepackage{float,amsmath}
\usepackage{amsfonts,epsfig}
\usepackage{subcaption}
\usepackage{color}
\usepackage{hyperref}
\usepackage[dvipsnames]{xcolor}
\usepackage{xargs}
\usepackage{verbatim}

\usepackage{tikz}
\usetikzlibrary{calc,arrows.meta,backgrounds,patterns,decorations.pathreplacing}
\usepackage{pgfplots}
\pgfplotsset{compat=1.18}

\usepackage{algorithm}
\usepackage{algpseudocode}

\newcommandx{\ilker}[2][1=]{\todo[size=tiny, linecolor=Green,backgroundcolor=Green!25,bordercolor=Green,#1]{
\begin{spacing}{0.75}
\tiny{\textcolor{blue}{#2}}
\end{spacing}}}
\newcommandx{\ilkercaption}[2][1=]{\todo[inline,linecolor=Green,backgroundcolor=Green!25,bordercolor=Green,#1]{\footnotesize{\textcolor{blue}{#2}}}}
\newcommandx{\wenhao}[2][1=]{\todo[size=tiny, linecolor=Orange,backgroundcolor=Orange!25,bordercolor=Orange,#1]{
\begin{spacing}{0.5}
\tiny{\textcolor{blue}{#2}}
\end{spacing}}}
\newcommandx{\wenhaocaption}[2][1=]{\todo[inline,linecolor=Orange,backgroundcolor=Orange!25,bordercolor=Orange,#1]{\footnotesize{\textcolor{blue}{#2}}}}

\newcommand{\addcite}[1]{\textcolor{blue}{[Ref.]}}

\title{Local Violation Certification for Linear \\ Predict-Then-Optimize Pipelines} 
\ShortTitle{Local Violation Certification}
\ShortAuthors{Birbil \& Chi}

\NumberOfAuthors{2}

\FirstAuthor{Ş. İlker Birbil}
\FirstAuthorAddress{Faculty of Economics and Business,
University of Amsterdam, The Netherlands}

\SecondAuthor{Wenhao Chi}
\SecondAuthorAddress{Yau Mathematical Sciences Center, Tsinghua University \\
Faculty of Economics and Business, University of Amsterdam, The Netherlands}


\allowdisplaybreaks

\begin{document}


\newcommand\hlb[3][]{ \todo[inline,caption={emptytext},
  size=\normalsize, backgroundcolor=yellow!70, bordercolor=yellow!70, noshadow, #1]{
    \begin{minipage}{
        \textwidth-4pt}#2
    \end{minipage}}
  \todo{\begin{spacing}{0.5}#3\end{spacing}}}

\newcommand{\hlc}[2]{\hl{#1}\todo{\begin{spacing}{0.5}#2\end{spacing}}}

\newcommand{\sitodo}[2][]{\todo[caption={#2}, #1]{
    \begin{spacing}{0.5}#2\end{spacing}}}

\newcommand{\intodo}[2][]{\todo[inline, noshadow, caption={emptytext}, #1]{{\bf \textcolor{blue}{TO-DO:}} \newline
    \begin{spacing}{1.0}\normalsize{#2}\end{spacing}}}

\newcommand{\sib}[1]{\textcolor{violet}{#1}}
\newcommand{\tsib}[1]{\begin{tcolorbox}\textcolor{violet}{#1}\end{tcolorbox}}

\newcounter{sibcmntcounter}
\setcounter{sibcmntcounter}{1}
\long\def\symbolfootnote[#1]#2{\begingroup
  \def\thefootnote{\fnsymbol{footnote}}\footnote[#1]{#2}\endgroup}
\newcommand{\sibcmnt}[1]{{\small\textbf{
      \textcolor{violet}{(C.\arabic{sibcmntcounter})}}
    \let\thefootnote\relax\footnotetext{\textcolor{violet}
        {\scriptsize(C.\arabic{sibcmntcounter})~#1}}}
  \addtocounter{sibcmntcounter}{1}}

\newcommand{\red}[1]{\textcolor{red}{#1}}
\newcommand{\blue}[1]{\textcolor{blue}{#1}}
\newcommand{\magenta}[1]{\textcolor{magenta}{#1}}
\newcommand{\rb}[1]{\raisebox{-1.5ex}[0cm][0cm]{#1}}
\newcommand{\HRule}{\noindent\rule{\linewidth}{0.5mm}}
\newcommand{\dsum}{\displaystyle\sum}
\newcommand{\veps}{\varepsilon}

\newcommand{\CA}{\mathcal{A}}
\newcommand{\CB}{\mathcal{B}}
\newcommand{\CC}{\mathcal{C}}
\newcommand{\CD}{\mathcal{D}}
\newcommand{\CG}{\mathcal{G}}
\newcommand{\CI}{\mathcal{I}}
\newcommand{\CJ}{\mathcal{J}}
\newcommand{\CK}{\mathcal{K}}
\newcommand{\CL}{\mathcal{L}}
\newcommand{\CM}{\mathcal{M}}
\newcommand{\CN}{\mathcal{N}}
\newcommand{\CP}{\mathcal{P}}
\newcommand{\CS}{\mathcal{S}}
\newcommand{\CT}{\mathcal{T}}
\newcommand{\CU}{\mathcal{U}}
\newcommand{\CX}{\mathcal{X}}
\newcommand{\YY}{\mathbb{Y}}
\newcommand{\ZZ}{\mathbb{Z}}
\newcommand{\RR}{\mathbb{R}}
\newcommand{\NN}{\mathbb{N}}
\newcommand{\II}{\mathbb{1}}

\renewcommand{\vec}[1]{{\boldsymbol{\mathbf{#1}}}}

\newcommand{\va}{\vec{a}}
\newcommand{\vb}{\vec{b}}
\newcommand{\vc}{\vec{c}}
\newcommand{\vd}{\vec{d}}
\newcommand{\ve}{\vec{e}}
\newcommand{\vf}{\vec{f}}
\newcommand{\vg}{\vec{g}}
\newcommand{\vh}{\vec{h}}
\newcommand{\vp}{\vec{p}}
\newcommand{\vr}{\vec{r}}
\newcommand{\vt}{\vec{t}}
\newcommand{\vu}{\vec{u}}
\newcommand{\vw}{\vec{w}}
\newcommand{\vx}{\vec{x}}
\newcommand{\vy}{\vec{y}}
\newcommand{\vz}{\vec{z}}
\newcommand{\zv}{\vec{0}}
\newcommand{\ov}{\vec{1}}

\newcommand{\vveps}{\vec{\veps}}
\newcommand{\veta}{\vec{\eta}}
\newcommand{\vxi}{\vec{\xi}}
\newcommand{\valpha}{\vec{\alpha}}
\newcommand{\vbeta}{\vec{\beta}}
\newcommand{\vgamma}{\vec{\gamma}}
\newcommand{\vtheta}{\vec{\theta}}
\newcommand{\vlambda}{\vec{\lambda}}
\newcommand{\vnu}{\vec{\nu}}
\newcommand{\vpi}{\vec{\pi}}
\newcommand{\vtau}{\vec{\tau}}
\newcommand{\vSigma}{\vec{\Sigma}}
\newcommand{\vOmega}{\vec{\Omega}}
\newcommand{\vTheta}{\vec{\Theta}}
\newcommand{\vmu}{\vec{\mu}}

\newcommand{\mA}{\vec{A}}
\newcommand{\mB}{\vec{B}}
\newcommand{\mC}{\vec{C}}
\newcommand{\mD}{\vec{D}}
\newcommand{\mE}{\vec{E}}
\newcommand{\mF}{\vec{F}}
\newcommand{\mG}{\vec{G}}
\newcommand{\mH}{\vec{H}}
\newcommand{\mI}{\vec{I}}
\newcommand{\mL}{\vec{L}}
\newcommand{\mM}{\vec{M}}
\newcommand{\mP}{\vec{P}}
\newcommand{\mQ}{\vec{Q}}
\newcommand{\mR}{\vec{R}}
\newcommand{\mS}{\vec{S}}
\newcommand{\mU}{\vec{U}}
\newcommand{\mV}{\vec{V}}
\newcommand{\mX}{\vec{X}}

\newcommand{\tr}{^{\top}}
\newcommand{\ntr}{^{-\top}}
\newcommand{\inv}{^{-1}}

\newcommand{\gr}{\mbox{graph}}
\newcommand{\ra}{\rightarrow}
\newcommand{\la}{\leftarrow}
\newcommand{\Ra}{\Rightarrow}
\newcommand{\rra}{\rightrightarrows}
\newcommand{\ptr}{\marginpar{$\Leftarrow$}}

\newcommand{\pfxi}{\frac{\partial f(\vx)}{\partial x_i}}
\newcommand{\pfx}{\partial f(\vx)}
\newcommand{\pf}{\partial f}
\newcommand{\pxi}{\partial x_i}
\newcommand{\px}{\partial x}

\newcommand{\nfx}{\nabla f(\vx)}
\newcommand{\eps}{\epsilon}
\newcommand{\eg}{\textit{e.g.}}
\newcommand{\ie}{\textit{i.e.}}

\newcommand{\vsp}{\vspace{4mm}}
\newcommand{\vspp}{\vspace{8mm}}
\newcommand{\vsppp}{\vspace{12mm}}

\newcommand{\hsp}{\hspace{4mm}}
\newcommand{\hspp}{\hspace{8mm}}
\newcommand{\hsppp}{\hspace{12mm}}

\newcommand{\pr}[1]{P\left(#1\right)}
\newcommand{\ex}[1]{\mathbb{E}\left[#1\right]}
\newcommand{\variance}[1]{\mbox{Var}\left(#1\right)}
\newcommand{\covar}[1]{\mbox{Cov}\left(#1\right)}
\newcommand{\C}[2]{\left(\begin{array}{c} #1 \\ #2 \end{array}\right)}

\newcommand{\maximize}{\mbox{maximize\hspace{4mm} }}
\newcommand{\minimize}{\mbox{minimize\hspace{4mm} }}
\newcommand{\subto}{\mbox{subject to\hspace{4mm}}}

\newenvironment{sibitemize}{
  \renewcommand{\labelitemi}{$\diamond$}
  \begin{itemize}
    \setlength{\parskip}{0mm}}
  {\end{itemize}}

\newcommand{\propnum}[2]{\vspace{3mm}
  \noindent {\sc Proposition #1}{\it #2} \vspace{3mm}}
\newcommand{\lemnum}[2]{\vspace{3mm}
  \noindent {\sc Lemma #1}{\it #2} \vspace{3mm}}
\newcommand{\thmnum}[2]{\vspace{3mm}
  \noindent {\sc Theorem #1}{\it #2} \vspace{3mm}}

\maketitle

\begin{abstract}
Data-driven decision pipelines combining predictive machine learning models with downstream optimization software are increasingly used to make high-stakes operational decisions. Certifying the safety, fairness, and reliability of these decisions is essential, yet traditional scenario generation methods rely on repeated random testing, which becomes computationally prohibitive when failure events are rare and offers little insight into why failures occur. We present a framework for local violation certification designed specifically for linear decision pipelines under input uncertainty. We mathematically demonstrate that standard sampling methods fail efficiently for rare violations, motivating a direct structural approach. By analyzing the fixed decision boundary of a deployed pipeline, we show that the local risk of failure can be calculated directly in closed form using a single optimization solve. Furthermore, we introduce an exact sampling procedure and closed-form risk statistics that provide feature-level attributions (identifying which input characteristics contribute most to potential non-compliance) without requiring repetitive random trials or complex sampling algorithms. We demonstrate our approach on an economic power dispatch system subject to emissions regulations, delivering precise, auditable risk assessments at a fraction of the traditional computational cost.
\end{abstract}

\section{Introduction.}
\label{sec:introduction}

Modern decision systems increasingly couple a machine learning predictive model with a large-scale constrained optimization program.  A trained regression model forecasts demand at each customer location; a linear program then allocates limited service capacity across locations while meeting regulatory service-level commitments.  A linear predictor estimates task-completion times for individual workers; an assignment program allocates tasks to groups while satisfying fairness constraints.  These \emph{Predict-Then-Optimize (PTO) pipelines} \citep{bertsimas2020predictive,ElmachtoubGrigas2022,DontiAmosKolter2017} are responsible for consequential decisions affecting large populations.

The deployment of such pipelines at scale raises an immediate question: do they respect fairness, robustness, or safety requirements?  Can we certify that a task-assignment pipeline does not systematically allocate the most burdensome work to a protected group -- not through any explicit rule, but because the predictor has absorbed historical patterns of inequity?  Can we certify that a resource-allocation pipeline will meet its service-level commitments for a minority customer segment under rare but plausible demand scenarios?  The EU AI Act \citep{EUAIA2024} now mandates conformity assessment for ML systems used in workforce management, critical infrastructure, and financial services (Annexes II--III, Articles 9, 10, 17). 

A natural approach to assessing pipeline reliability is \emph{scenario generation} \citep{shapiro2021lectures}: draw a large number of input samples independently from the input distribution, run each through the pipeline, and record the fraction that produce a violation. This approach is simple, model-agnostic, and widely used in practice. However, it suffers from two fundamental limitations that become severe precisely when reliability matters most--that is, when violations are rare. First limitation is \emph{prohibitive sample complexity}. When violations are rare events, the fraction of independent identically distributed (i.i.d.) samples that land in the violation region is small. To estimate the violation rate to a fixed absolute accuracy with high probability, one needs a number of samples that is inversely proportional to the violation rate, which we show this formally in our work. This is prohibitive when the violation rate is small. For instance, certifying that a pipeline disproportionately affects a minority group comprising a small fraction of the population requires on the order of thousands of pipeline evaluations, each solving an optimization problem just to observe a handful of violations. The second limitation is \emph{no characterization of the violation distribution}. Even when enough samples are collected to estimate the violation rate, the scenario generation approach provides no information about the \emph{structure} of violations: which inputs are most likely to cause a violation, how concentrated violations are in the input space, or what distinguishes marginal violators from severe ones. This information is essential for two downstream tasks that practitioners pay attention to:
\begin{enumerate}
    \item \emph{Explanation.} Given that a violation occurs, which input characteristics are responsible? Answering this requires the conditional distribution of the input given a violation, which scenario generation cannot provide efficiently when violations are rare.
    
    \item \emph{Uncertainty quantification.} How robust is the decision to input uncertainty? Is the violation region a thin sliver near a decision boundary, or a broad set that captures a substantial portion of the plausible inputs? The geometry of the violation distribution determines the answer, and cannot be inferred from a violation rate alone.
\end{enumerate}

In this paper, we study the local version of the reliability question. The practical setting we have in mind is a single committed decision. A \emph{deployed pipeline} observes an input and implements the decision computed from it. The input is, however, subject to uncertainty (measurement noise, estimation error, or a prescribed stress model) so the realized input may be perturbed. We model the deployed system as re-solving the mathematical model at the realized input, so the implemented decision is the perturbed optimal solution. Then, we ask: how robust is the decision to perturbations of the given input? What fraction of perturbed inputs trigger a violation? What typical violating perturbations look like? We show for a linear predict-then-optimize pipeline (affine predictor, linear programming model, linear violation condition) that under Gaussian perturbations, these questions have a complete and computationally very efficient answer: the local violation rate enjoys a closed-form evaluation and the local violation distribution admits an exact direct sampler. This local characterization is not merely a computational convenience. It directly addresses what reliability assessment of a deployed system requires: a certificate for the specific decision as deployed and the most-sensitive perturbation direction providing feature-level attributions.

\paragraph{Related Work.} Our setting is the \emph{predict-then-optimize} pipeline, in which a learned model supplies the parameters of a downstream optimization problem. A large literature studies how to \emph{train} such pipelines so that predictions serve the eventual decision: the smart ``predict, then optimize'' loss of \citep{ElmachtoubGrigas2022}, task-based end-to-end learning through the optimization layer \citep{DontiAmosKolter2017}, and the prescriptive-analytics framework of \citep{bertsimas2020predictive}. That work asks how to \emph{build} a better pipeline; we take a trained pipeline as given and ask
a different question (how reliable is the decision it has already committed to) which is an auditing rather than a training problem.

A second body of work confronts uncertainty at optimization time. Stochastic programming optimizes an expected objective over a distribution of scenarios \citep{shapiro2021lectures}, and chance-constrained programming requires that constraints hold with high probability, replacing an intractable probabilistic constraint by a tractable convex surrogate \citep{nemirovski2006convex}. These methods \emph{design} a decision that is feasible or near-optimal with a prescribed probability; by contrast, we do not design a decision but certify, in closed form, the violation probability of one that a deployed pipeline has already produced. Where chance-constrained programming builds conservatism into the solution ex ante, we measure exposure ex post.

Closest in spirit is the \emph{scenario approach} to robust and chance-constrained design \citep{calafiore2006scenario, campi2008exact}, which samples constraints and bounds the residual probability that the resulting decision is infeasible. Its guarantees, like ours, concern a violation probability, but they hold \emph{by construction} of the decision and improve with the sample size, whereas our certificate evaluates an \emph{already-fixed} decision with no sampling at all. Robust optimization likewise immunizes a decision against an entire uncertainty set \citep{bental1998robust, bertsimas2004price}, trading a probabilistic statement for a worst-case guarantee. In all of these, uncertainty enters when the decision is \emph{made}. However in ours, the decision is made and we assess it afterward. This audit-versus-design distinction is what lets the local rate admit a closed form: the decision, and hence the geometry of the violation set, is fixed before the analysis begins.

Recent studies have explored generative modeling methods for learning data distributions, particularly to improve representations of rare events, extreme samples, and tail behaviors. For example, ExGAN first employs a distribution shifting approach to increase the proportion of extreme samples in the training data, and then uses Extreme Value Theory to model the shifted data distribution, from which the generator learns to produce realistic and extreme samples \citep{bhatia2021ExGANA}. Tail-GAN focuses on learning financial scenario distributions that preserve critical tail properties, namely Value at Risk (VaR) and Expected Shortfall (ES). By leveraging the joint elicitability of the (VaR, ES) pair, Tail-GAN uses a strictly consistent scoring function as an adversarial objective, encouraging the generator to produce scenarios that reproduce the tail behavior of the training data distribution \citep{cont2026TailGAN}. Although these approaches are effective in generating rare-event scenarios, they fundamentally differ from our objective. ExGAN relies on Extreme Value Theory to characterize the distribution of extreme regions and trains a generator to produce extreme samples. 
Tail-GAN obtains the rare events from the learned generator through scenario generation. However, our approach directly characterize and sample from the local violation distribution associated with the PTO pipeline.

Beyond generative modeling, uncertainty quantification has also been investigated in optimization. \citet{chanU2026Uncertainty} proposed a hierarchical Bayesian framework in data-driven inverse optimization. Rather than providing only point estimates for the unknown parameters, their approach characterized parameter uncertainty through posterior distributions and constructed credible regions around the estimates. The authors further developed two Markov chain Monte Carlo algorithms to sample from the distribution under two data-generating processes, enabling practical implementation. Although both their work and ours aim to characterize the uncertainty through distributional representation, they focus on uncertainty arising from model parameters. In contrast, our work considers uncertainty induced by perturbations in the input space and aims to characterize the conditional distribution of inputs that lead to violation. 


With our current work, we make the following contributions to the literature:

\begin{sibitemize}
    \item We prove that any method based on independent random samples needs, on average, a number of samples that grows in inverse proportion to the violation rate in order to reliably detect a violation. Rare violations are therefore prohibitively expensive to certify by sampling alone, which motivates the structural approach that we study.
    \item For a fixed deployed input subject to Gaussian forecast uncertainty with an arbitrary covariance, we show that the local violation rate has a closed form. It is obtained by evaluating the standard normal distribution at a single distance that measures how far the deployed decision sits from the violation boundary, relative to the shape of the uncertainty. Computing it requires only one optimization solve and one scalar Gaussian evaluation, with an explicitly bounded approximation error.
    \item The distribution of violating inputs admits an exact sampler that draws independent samples directly with no Markov chain machinery, burn-in, thinning, or boundary handling. The direction of greatest sensitivity yields feature-level attribution, and together the distance-to-boundary and the violation rate form an individual, per-decision violation certificate. 
    \item We demonstrate the method on an illustrative application example of economic-dispatch pipeline under an emissions cap. This application shows that the violation certificate is derived end-to-end from the cost structure and the active dispatch basis rather than posited, and that it yields an interpretable, per-decision compliance audit.
\end{sibitemize}

\section{Problem Setting.} 
\label{sec:problem_setting}

Let $\CX \subseteq \RR^p$ denote the input space and $X$ be a
random input drawn from a distribution $P$ supported on $\CX$.  A \emph{PTO pipeline} is a composition
\[
    X ~ \xrightarrow{\text{predictor}} ~ \theta(X) ~ \xrightarrow{\text{optimizer}} ~ z(\theta(X)),
\]
where $\theta : \CX \to \RR^d$ is a predictive model and $z : \RR^d \to \RR^m$ is the argmin map of a parametric mathematical program. 

\subsection{Violation Condition.}
\label{subsec:violation}

A \emph{violation condition} is a measurable predicate $V : \CX \to \{0,1\}$ that flags whether the pipeline's decision on input $x$ is undesirable:
\begin{equation}
    \label{eq:violation-indicator}
    V(x) = \mathbf{1}\bigl[g(x, z(\theta(x))) \geq 0\bigr],
\end{equation}
where $g : \CX \times \RR^m \to \RR$ is a function encoding the violation criterion. Typical examples include a resource allocation decision exceeding available capacity or a decision variable falling outside a prescribed range designated by regulations. The \emph{violation set} then becomes set of inputs that trigger a violation:
\begin{equation}
    \label{eq:violation-set}
    S = \bigl\{x \in \CX : V(x) = 1\bigr\} 
    = \bigl\{x \in \CX : g(x, z(\theta(x))) \geq 0\bigr\}.
\end{equation}

At this point one might ask why the violation condition is not simply added as an additional constraint to the downstream optimization problem, thereby eliminating violations by construction. There are several reasons why this is often not possible or not appropriate. First, the violation condition is typically defined \emph{externally} to the pipeline (by a regulator, auditor, or affected group) and may be imposed after the pipeline is already in production. The auditor does not have the authority or access to modify the optimization problem. Second, adding the condition as a hard constraint may render the problem \emph{infeasible} for some inputs, or degrade the objective in ways the operator cannot accept. Third, and most fundamentally, the goal is not to eliminate violations but to \emph{characterize their distribution}: which inputs are most at risk, how concentrated violations are, and what structural features of the pipeline drive them. This diagnostic information is precisely what informs whether and how the pipeline should be revised, and forms the evidence base for remedial or regulatory action. Finally, the condition may depend on \emph{sensitive attributes} (for example in fairness settings) that are legally excluded from the optimization, so that the optimizer cannot be constrained on variables it does not observe, even though the condition is defined in terms of their effect on outcomes.

\subsection{Violation Rate and Violation Distribution.}
\label{subsec:violation-dist}

The \emph{violation rate} is the probability that a randomly drawn input triggers a violation:
\begin{equation}
    \label{eq:violation-rate}
    \nu := P(S) = \mathbb{E}_P[V(X)].
\end{equation}
The violation rate is the quantity estimated by scenario generation. While the violation rate $\nu$ characterizes the prevalence of violations, the certification goals require characterizing the \emph{violation distribution}. Thus, the object of central interest in this paper is the conditional distribution of the input given that a violation occurs. That is,
\begin{equation}
    \label{eq:violation-dist}
    Q := P(\cdot \mid X \in S).
\end{equation}
The violation distribution $Q$ characterizes the full geometry of the violation region as seen through the lens of $P$: it assigns high probability to inputs that are both likely under $P$ and deep inside $S$, and low probability to inputs that are unlikely or only marginally violating. Samples from $Q$ can be used directly for explanation and uncertainty quantification, the two tasks that were first described in Section \ref{sec:introduction}:
\begin{enumerate}
    \item \emph{Explanation. } Produce a set of samples $x_1, \dots, x_N \overset{\text{approx}}{\sim} Q$. These samples answer questions of the form: \emph{what does a typical violating input look like}, and \emph{which input features are most predictive of a violation?}
    \item \emph{Uncertainty quantification.} Characterize the \emph{concentration} and \emph{spread} of the violation distribution $Q$ by estimating summary statistics such as the mean and the covariance of violating inputs. These quantities answer questions of the form: \emph{how concentrated are violations in the input space}, and \emph{are violations driven by a single feature or a combination of features?}
\end{enumerate}

The two tasks above can be pursued at two complementary levels of granularity. In \emph{global} violation characterization, $P$ is the full population distribution and $Q = P(\cdot \mid X \in S)$ describes systemic risk across the entire input space. In \emph{local} violation characterization, $P$ is replaced by a perturbation distribution centered at a specific reference input $x_0$, and the violation distribution describes the robustness of the pipeline's decision at $x_0$ to input uncertainty. The two levels correspond to two distinct certification actors: a regulator assessing aggregate behavior across a population, and an individual (or their advocate) assessing the fragility of a specific decision. This paper focuses on the \emph{local} characterization, for which both certification tasks above admit a closed-form solution with no sampling. 

\section{Hardness of Scenario Generation.}
\label{sec:hardness}

Scenario generation estimates the violation rate $\nu$ by drawing $M$ independent samples $x_1, \dots, x_M \overset{\text{i.i.d.}}{\sim} P$ and computing the sample fraction
\begin{equation}
    \label{eq:sg-estimator}
    \hat{\nu}_M := \frac{1}{M} \sum_{i=1}^M V(x_i)
    = \frac{1}{M} \sum_{i=1}^M \mathbf{1}[x_i \in S].
\end{equation}
This is the natural unbiased estimator of $\nu$, and it is the basis of virtually all practical compliance checks in predict-then-optimize pipelines. We now show that this estimator is fundamentally limited when violations are rare.

We measure the quality of an estimator by its ability to distinguish a pipeline with no violations from one with a small but nonzero violation rate. Formally, consider two hypotheses:
\begin{align}
    H_0 &: \nu = 0 \quad \text{(no violations)}, \label{eq:H0} \\
    H_1 &: \nu = \nu_0 \quad \text{(violation rate } \nu_0 > 0\text{)}, \label{eq:H1}
\end{align}
for some $\nu_0 > 0$ to be specified. Any useful certification procedure must be able to distinguish $H_0$ from $H_1$ with non-trivial probability.

\begin{theorem}[Sample complexity lower bound]
\label{thm:lower-bound}
Let $\hat{\nu}_M$ be any estimator of $\nu$ based on $M$ i.i.d.\ samples from $P$. For any $\nu_0 \in (0, 1/2]$ and any $\delta \in (0, 1/2)$, if
\begin{equation}
    \label{eq:M-condition}
    M < \frac{\log(1/2\delta)}{2\nu_0},
\end{equation}
then there exist two distributions $P_0$ and $P_1$ over $\CX$ with violation rates $0$ and $\nu_0$ respectively, such that no estimator based on $M$ i.i.d. samples can distinguish $H_0$ from $H_1$ with probability greater than $1 - \delta$. In particular, $\Omega(1/\nu_0)$ samples are necessary to detect a violation rate of $\nu_0$ with constant probability.
\end{theorem}

\begin{proof}
Under $H_0$, no sample falls in $S$, so $\hat{\nu}_M = 0$ with probability one. Under $H_1$, each sample falls in $S$ independently with probability $\nu_0$, so the probability that \emph{no} sample falls in $S$ is $(1 - \nu_0)^M$. We lower-bound this quantity in two steps: (i) For $t \in [0,1)$, the inequality $\log(1-t) \geq -t/(1-t)$ holds. Applying this with $t = \nu_0$ and multiplying by $M > 0$ yields
\begin{equation}
\label{eq:step1}
(1 - \nu_0)^M \geq \exp\left(-\frac{M\nu_0}{1-\nu_0}\right).
\end{equation}
(ii) Since $\nu_0 \leq 1/2$, we have $1 - \nu_0 \geq 1/2$, so $1/(1-\nu_0) \leq 2$, and therefore we obtain
\begin{equation}
\label{eq:step2}
\exp\left(-\frac{M\nu_0}{1-\nu_0}\right) \geq \exp(-2M\nu_0).
\end{equation}
Combining now \eqref{eq:step1} and \eqref{eq:step2} leads to
\begin{equation}
    \label{eq:no-violation}
    P_1\left(\hat{\nu}_M = 0\right) = (1 - \nu_0)^M \geq \exp(-2M\nu_0).
\end{equation}
If $M < \log(1/2\delta)/(2\nu_0)$, then $2M\nu_0 < \log(1/2\delta)$, so $\exp(-2M\nu_0) > \exp(-\log(1/2\delta)) = 2\delta$, and therefore $P_1(\hat{\nu}_M = 0) > 2\delta$.

Since $\hat{\nu}_M = 0$ under $H_0$ with probability one and under $H_1$ with probability greater than $2\delta$, any decision rule that declares $H_0$ when $\hat{\nu}_M = 0$ incurs error probability greater than $2\delta$ under $H_1$, and any rule that declares $H_1$ when $\hat{\nu}_M = 0$ incurs error probability one under $H_0$. Therefore no estimator can achieve error probability below $\delta$ under both hypotheses simultaneously, establishing the lower bound. The asymptotic statement $\Omega(1/\nu_0)$ follows immediately from \eqref{eq:M-condition}.
\end{proof}

Theorem \ref{thm:lower-bound} formalizes the first limitation of scenario generation. We now state aforementioned two limitations precisely, as they motivate our approach in the subsequent sections:
\begin{enumerate}
    \item \emph{Prohibitive sample complexity.} By Theorem \ref{thm:lower-bound}, any i.i.d.-based approach requires $\Omega(1/\nu_0)$ samples to reliably detect violations. This is a \emph{fundamental} limitation: it applies to any estimator based on i.i.d. samples from $P$, not just the sample fraction $\hat{\nu}_M$. The only way to overcome it is to use a \emph{non-i.i.d.} sampling strategy that actively guides samples into the violation region $S$.
    
    \item \emph{No characterization of the violation distribution.} Even if $M$ is large enough to observe violations, the samples $\{x_i : x_i \in S\}$ collected by scenario generation are distributed according to $Q = P(\cdot \mid X \in S)$ only in the limit $M \to \infty$. For finite $M$, the expected number of violation samples is $M\nu_0$, which is $\mathcal{O}(1)$ when $M = \mathcal{O}(1/\nu_0)$. A handful of violation samples is insufficient to characterize the violation distribution $Q$ and estimate its mean, covariance, or marginals to any useful accuracy.
\end{enumerate}

It is tempting to hope that the barrier is an artifact of asking a population-level question, and that restricting attention to a single deployed input makes the problem easy. It does not. Consider the \emph{local} reliability question: fix a reference input $x_0$ (the operating point of the pipeline as deployed) and model input uncertainty by a perturbation distribution $P_{\mathrm{local}}$ centered at $x_0$ (defined formally in Section \ref{sec:local}). The quantity of interest is the local violation rate $P_{\mathrm{local}}(X \in S)$. Since Theorem \ref{thm:lower-bound} holds for \emph{any} input distribution, it applies verbatim to $P_{\mathrm{local}}$: detecting a local violation rate of $\nu_0$ by scenario generation still requires $\Omega(1/\nu_0)$ perturbed evaluations, and local violations are typically rare precisely when the deployed decision is robust (the regime an auditor most wants to certify). Localizing the question therefore does not, by itself, rescue sampling where each such draw is a full pipeline evaluation, i.e., an optimization solve. Yet, our subsequent discussion shows that the local question nonetheless can admit a closed-form answer that sidesteps sampling entirely. This is achieved by exploiting the structure of the linear pipeline rather than probing it with draws.

\section{Linear Pipeline and Optimal-Solution Geometry.}
\label{sec:linear}

We specialize to a setting in which all three components of the pipeline, i.e., the predictor, the optimizer, and the violation condition, are linear. This covers a range of practically relevant pipelines in operations research, and it is the setting in which the geometry of the optimal-solution map can be described exactly.

We assume that the predictor is an affine map of the form
\begin{equation}
    \label{eq:linear-predictor}
    \theta(x) = Bx + b,
\end{equation}
where $B \in \RR^{d \times p}$ is a weight matrix and $b \in \RR^d$ is a bias vector. Affine predictors arise naturally as ordinary least squares regressors, ridge regressors, or their weighted  ensembles. Given the predicted parameter vector $\theta(x)$, the optimizer solves the linear program (LP):
\begin{equation}
    \label{eq:lp}
    z(x) := \operatorname*{arg min}_{z \in \RR^m}  c^\top z
    \quad \text{subject to} \quad Az \le \theta(x),
\end{equation}
where $c \in \RR^m$ is a fixed objective vector and $A \in \RR^{d \times m}$ is a fixed constraint matrix. The right-hand side of the constraints is the predicted parameter vector $\theta(x)$, so the feasible region varies with the input $x$ through the predictor. This formulation captures a broad class of resource allocation, scheduling, and routing problems in which the constraint capacities or demands are predicted from data. 

To obtain the local characterization, we next focus our attention to a deployed pipeline observing a fix reference input $x_0$ and implementing the decision $z(x_0)$.

\begin{assumption}[Boundedness and non-degeneracy]
\label{ass:nondeg}
At the reference input $x_0 \in \CX$, the LP \eqref{eq:lp} is bounded below and attains a unique optimal solution $z(x_0)$ at a non-degenerate vertex of the feasible polyhedron $\{z : Az \le \theta(x_0)\}$.
\end{assumption}

The boundedness clause is not automatic: the variables $z$ are free and the constraints are one-sided, so the LP \eqref{eq:lp} can be unbounded unless the constraint normals (the rows of $A$) positively span $\RR^m$. Assumption \ref{ass:nondeg} rules this out at $x_0$ and is all the local analysis requires; we state it point-wise rather than for every $x \in \CX$. Now, under Assumption \ref{ass:nondeg}, the optimal solution is determined by a unique
active \emph{basis} $\mathcal{B}(x) \subseteq \{1,\dots,d\}$ with
$|\mathcal{B}(x)| = m$, and is the basic feasible solution
\begin{equation}
    \label{eq:bfs}
    z(x)  =  A_{\mathcal{B}(x)}^{-1} \theta_{\mathcal{B}(x)}(x),
\end{equation}
where $A_{\mathcal{B}(x)}$ and $\theta_{\mathcal{B}(x)}$ are the submatrix and subvector of
$A$ and $\theta$ indexed by $\mathcal{B}(x)$. 

We consider a violation condition that is also linear in the optimal decision $z(x)$:
\begin{equation}
    \label{eq:linear-violation}
    V(x) = \mathbf{1}\bigl[w^\top z(x) \geq \gamma\bigr],
\end{equation}
where $w \in \RR^m$ is a fixed weight vector and $\gamma \in \RR$ is a threshold. This covers the fairness, capacity, and regulatory examples of
Section \ref{subsec:violation} when the violation criterion is linear in the decision. Below we give two examples to illustrate linear pipelines. 

\begin{example}[Task assignment with group-fairness compliance]
  \label{ex:fairness}
  A predictive model estimates task-completion times $\theta_i(x) = (Bx+b)_i$ for each worker $i$ from observable features $x$. An assignment LP minimizes total completion time subject to capacity constraints $A z \le \theta(x)$, with LP relaxation having integral optimal solutions for standard assignment matrices.  Let $D \in \RR^{m \times m}$ be the burden matrix mapping decisions to individual workloads. Define $u_r = |g|\inv \sum_{i \in g} D_i\tr  \in \RR^m$ as the average burden row for group $r$, and $u = n\inv \sum_{i=1}^m D_i\tr  \in \RR^m$ as the overall average burden row, so that $\bar{b}_r(z) = u_r\tr  z$ and $\bar{b}(z) = u\tr  z$. The compliance requirement is \emph{demographic parity}: the mean task-burden assigned to each protected group $g$ does not exceed the overall mean by more than $\tau > 0$:
  \[
    V_{\mathrm{fair}}(x) = \mathbf{1}\left[\max_r (u_r - u)\tr  z(x) \geq \tau \right].
  \]
\end{example}

\begin{example}[Resource allocation with minority service-level compliance]
  \label{ex:robust}
  A predictive model estimates demand $\theta_i(x)$ at each of $m$ customer locations.  An LP allocates limited supply subject to $A z \le \theta(x)$. Define $e_\CM = \sum_{i \in \CM} e_i \in \RR^m$ as the indicator vector for the minority segment $\CM \subset [m]$, where $e_i$ is the $i$-th standard basis vector.  The compliance requirement is a \emph{service-level guarantee}: the total allocation to $\CM$ must cover at least fraction $\rho$ of its demand:
  \[
    V_{\mathrm{level}}(x) = \mathbf{1}\left[e_\CM\tr  z(x) \le \rho e_\CM\tr \theta(x)\right].
  \]
\end{example}

The combination of the affine predictor \eqref{eq:linear-predictor} and the LP \eqref{eq:lp} makes the optimal solution $z(x)$ affine in a neighborhood of any fixed input. This single-region fact is all the local analysis requires.

\begin{lemma}[Affine structure]
\label{lem:local-affine}
Under Assumption \ref{ass:nondeg}, fix the reference input $x_0 \in \CX$ and let $\mathcal{B}_{0} := \mathcal{B}(x_0)$ be the optimal basis at $x_0$. Then, there is a polyhedron $\CX_0$ with $x_0 \in \CX_0$ on which $\mathcal{B}_0$ remains optimal, and the optimal solution is affine in $x$:
\begin{equation}
    \label{eq:pwa}
    z(x) = C_0 x + d_0, \qquad x \in \CX_0,
\end{equation}
where $C_0 := A_{\mathcal{B}_0}^{-1} B_{\mathcal{B}_0} \in \RR^{m \times p}$ and $d_0 := A_{\mathcal{B}_0}^{-1} b_{\mathcal{B}_0} \in \RR^m$. Moreover non-degeneracy at $x_0$ places $x_0$ in the \emph{interior} of $\CX_0$, so $\CX_0$ is a neighborhood of $x_0$.
\end{lemma}

\begin{proof}
We first describe the set of right-hand sides $\theta$ for which $\mathcal{B}_0$ remains optimal. Optimality of a basis is the conjunction of primal feasibility and dual feasibility. The dual-feasibility (reduced-cost) conditions depend only on the cost $c$ and the matrix $A$, both fixed here; they do \emph{not} involve $\theta$, so they hold throughout once they hold at $x_0$. Optimality therefore reduces to primal feasibility of $\mathcal{B}_0$, namely that the basic solution $A_{\mathcal{B}_0}^{-1}\theta_{\mathcal{B}_0}$ satisfies the $d-m$ non-basic constraints $A_i z \le \theta_i$, $i \notin \mathcal{B}_0$ (the $m$ basic constraints hold with equality by construction). Substituting the basic solution into each of these, primal feasibility requires $A_i A_{\mathcal{B}_0}^{-1}\theta_{\mathcal{B}_0} \le \theta_i$ for $i \notin \mathcal{B}_0$, a single linear inequality in $\theta$. Stacking these $d-m$ inequalities gives a matrix $M \in \RR^{(d-m)\times d}$, and being homogeneous in $\theta$ they define a polyhedral cone $\{\theta : M\theta \le 0\}$ in $\RR^d$.

The pre-image of a polyhedron under an affine map is again a polyhedron: substituting $\theta(x) = Bx + b$ into $M\theta \le 0$ yields the polyhedron $\CX_0 := \{x \in \RR^p : (MB) x \le -Mb\}$. Since $\mathcal{B}_0$ is optimal at $x_0$ and, by non-degeneracy, the primal-feasibility inequalities hold there \emph{strictly}, $x_0$ lies in the interior of $\CX_0$, which is hence a neighborhood of $x_0$.

On $\CX_0$ the basis is fixed at $\mathcal{B}_0$, so substituting $\theta(x) = Bx + b$ into the basic feasible solution \eqref{eq:bfs} gives 
\[
z(x) = A_{\mathcal{B}_0}^{-1}\theta_{\mathcal{B}_0}(x)
= A_{\mathcal{B}_0}^{-1}\bigl(B_{\mathcal{B}_0} x + b_{\mathcal{B}_0}\bigr)
= C_0 x + d_0,
\]
which is \eqref{eq:pwa}.
\end{proof}

Lemma \ref{lem:local-affine} is a statement about a \emph{deterministic} input: on the neighbourhood $\CX_0$ the optimal solution is the affine map $z(x) = C_0 x + d_0$. Once we perturb the reference input in Section \ref{sec:local}, this map turns the random input into a random optimal solution $Z_0(X) = C_0 X + d_0$, whose conditional law given a violation is the object we characterize.

\section{Violation Characterization.}
\label{sec:local}

We now develop the central result for local violation certification. We shall next show for the operating point of a deployed pipeline (i.e., a fixed reference input) and a Gaussian model of input uncertainty that the local violation rate and the local violation distribution are available in closed form.

\subsection{Perturbation Model.} 
\label{subsec:local-setup}

Let $x_0 \in \CX$ be a fixed reference input, and model perturbations as
\begin{equation}
    \label{eq:local-dist}
    X = x_0 + \xi, \quad \xi \sim \mathcal{N}(0, \Sigma), \quad \Sigma \succ 0,
\end{equation}
so that the local input distribution is $P_{\mathrm{local}} = \mathcal{N}(x_0, \Sigma)$, i.e., Gaussian with mean $x_0$ and covariance $\Sigma$. Here, $\Sigma$ encodes the scale \emph{and} the directional structure of the input uncertainty (measurement noise, estimation error, or a prescribed stress model). 

Under this perturbation, the optimal solution becomes a \emph{random vector}
\begin{equation}
    \label{eq:Z0}
    Z_0(X) := \arg\min_{z \in \RR^m} ~ c^\top z
    \quad \text{subject to} \quad Az \le BX + b.
\end{equation}
Then, the violation set in input space becomes
\begin{equation}
    \label{eq:Sdef}
    S = \bigl\{x \in \CX : w^\top Z_0(x) \ge \gamma\bigr\},
\end{equation}
which leads to the violation event $\{X \in S\} \equiv \{w^\top Z_0(X) \ge \gamma\}$. The \emph{local violation distribution} is the conditional law of the input given a violation:
\begin{equation}
    \label{eq:local-viol-dist}
    Q_{\mathrm{local}} := P_{\mathrm{local}}(\cdot \mid X \in S)
    = \mathcal{N}(x_0, \Sigma)(\cdot \mid X \in S),
\end{equation}
which is a distribution on input space $\RR^p$. This is the object we shall characterize.

On the basis region $\CX_0$ of Lemma \ref{lem:local-affine} the map \eqref{eq:Z0} is the affine piece $Z_0(X) = C_0 X + d_0$. Thus, the random optimal solution itself is also Gaussian over the same region:
\begin{equation}
    \label{eq:Z0-gaussian}
    Z_0(X) \sim \mathcal{N}\bigl(C_0 x_0 + d_0, ~C_0 \Sigma C_0^\top\bigr)
    ~ \text{on } \CX_0.
\end{equation}
The scalar violation margin then reduces to $w^\top Z_0(X) \sim \mathcal{N}\bigl(w^\top(C_0 x_0 + d_0),w^\top C_0 \Sigma C_0^\top w\bigr)$. This is exactly the scalar Gaussian that drives the closed-form rate in the next section.

\subsection{The Active Basis Region.}
\label{subsec:local-component}

Solving the LP \eqref{eq:lp} at $\theta(x_0) = Bx_0 + b$ yields a unique optimal basis $\mathcal{B}(x_0)$ under Assumption \ref{ass:nondeg}, and hence identifies a polyhedral neighborhood $\CX_0$ via Lemma \ref{lem:local-affine}. The Gaussian $\mathcal{N}(x_0, \Sigma)$ concentrates around $x_0$, and the contribution of all other basis regions to the local violation distribution is controlled by the probability of leaving $\CX_0$. Using the definition of $\CX_0$ from the proof of Lemma \ref{lem:local-affine}, we define $n_j^\top := (MB)_j$ for the $j$-th row of $MB$ and $h_j := (-Mb)_j$ for the $j$-th component of $-Mb$. Then, we list the facets as $\CX_0 = \{x \in \RR^p: n_j^\top x \le h_j,\ j = 1,\dots,J\}$ with $J \le d-m$ and observe that $x_0$  resides on the interior of $\CX_0$. For each facet, the scalar $n_j^\top X \sim \mathcal{N}(n_j^\top x_0, n_j^\top \Sigma n_j)$, so the \emph{Mahalanobis margin} of $x_0$ to facet $j$ is $\rho_j(x_0) := (h_j - n_j^\top x_0)/\sqrt{n_j^\top \Sigma n_j}$. 

Leaving $\CX_0$ means violating at least one facet, so the exit event is the union $\{X \notin \CX_0\} = \bigcup_{j=1}^{J}\{n_j^\top X > h_j\}$. A union bound gives $P_{\mathrm{local}}(X \notin \CX_0) \le \sum_{j=1}^{J} P_{\mathrm{local}}(n_j^\top X > h_j)$, and standardizing the univariate Gaussian projection yields $P_{\mathrm{local}}(n_j^\top X > h_j) = \Phi(-\rho_j(x_0))$. Combining with the Gaussian tail inequality, we obtain
\begin{equation}
    \label{eq:exit-prob}
    P_{\mathrm{local}}\bigl(X \notin \CX_0\bigr) 
    \leq  \sum_{j=1}^{J} \Phi \bigl(-\rho_j(x_0)\bigr)
    \leq  J \exp \left(-\tfrac{1}{2} \rho_{\min}(x_0)^2\right),
\end{equation}
where $\rho_{\min}(x_0) := \min_{1 \le j \le J} \rho_j(x_0)$. The first inequality is not an equality because the facet-violation events overlap. It is tightest when the nearest facet dominates, the regime quantified by $\rho_{\min}$. When every facet is $\Sigma$-standard-deviations away (i.e., $\rho_{\min}(x_0)$ large), this probability is negligible. We formalize this as a working assumption.

\begin{assumption}[Small perturbation]
\label{ass:small-sigma}
The covariance $\Sigma$ is such that the smallest Mahalanobis margin satisfies $\rho_{\min}(x_0) \ge \sqrt{2\log(J/\varepsilon_0)}$ for a prescribed tolerance $\varepsilon_0 \in (0,1)$, so that, by \eqref{eq:exit-prob}, $P_{\mathrm{local}}(X \notin \CX_0) \leq \varepsilon_0$. When $\Sigma = \sigma^2 I$ (isotropic case), this reads $\sigma \le \delta(x_0)/\sqrt{2\log(J/\varepsilon_0)}$, where $\delta(x_0) = \min_j (h_j - n_j^\top x_0)/\|n_j\|$ is the Euclidean distance from $x_0$ to the nearest facet.
\end{assumption}

The bound \eqref{eq:exit-prob} is a union bound over the $J \le d-m$ facets of $\CX_0$. It is loose when facets are nearly parallel or far from $x_0$, in which case the admissible covariance is correspondingly larger than the sufficient condition suggests. The binding quantity is the \emph{nearest} facet, through $\rho_{\min}(x_0)$: Assumption \ref{ass:small-sigma} asks that $x_0$ be comfortably interior to its basis region, measured in the $\Sigma$-metric. The required margin grows only like $\sqrt{\log(J/\varepsilon_0)}$, so for the \emph{rate}, which incurs an additive $\mathcal{O}(\varepsilon_0)$ error, the assumption is mild in non-degenerate instances: a margin of a few standard deviations suffices, regardless of how rare violations are. It binds only near vertex degeneracy, where the basis region is thin.

This degenerate regime is, however, exactly the one in which the assumption \emph{should} bind. Let $d^*_\Sigma(x_0)$ denote the Mahalanobis signed distance from $x_0$ to the violation boundary. When $\rho_{\min}(x_0) \lesssim d^*_\Sigma(x_0)$, a perturbation reaching the violation boundary must cross a basis-change facet, so the re-solved optimal decision $z(X)$ moves to a neighboring vertex with its own affine piece and violation half-space. The committed decision $z(x_0)$ is unchanged, but the decision \emph{at the perturbed input} switches. The violation distribution is then a genuine mixture over basis regions which no single-region characterization, closed-form or sampled, can capture exactly. Assumption \ref{ass:small-sigma} thus delimits precisely the regime in which a single-region treatment is the right object, rather
than restricting the method within it.

Under Assumption \ref{ass:small-sigma}, the local violation distribution is the single-component conditional up to a controlled error (quantified formally in the next two sections): 
\begin{equation}
    \label{eq:local-Qk0}
    Q_{\mathrm{local}}  = \mathcal{N}(x_0, \Sigma)\bigl(\cdot \mid X \in S_0\bigr)
     +  (\text{error of order } \varepsilon_0),
\end{equation}
where restricting the linear violation condition \eqref{eq:linear-violation} to the region $\CX_0$ via \eqref{eq:pwa} leads the local violation set to the half-space intersection
\begin{equation}
    \label{eq:Sk0}
    S_0  =  \bigl\{x \in \CX_0 : a^\top x \geq b_0\bigr\},
    \qquad
    a  :=  (w^\top C_0)^\top \in \RR^p, \quad
    b_0  :=  \gamma - w^\top d_0.
\end{equation}

With \eqref{eq:Sk0}, we have a complete polyhedral description of the violation set; i.e., the half-space $\{a^\top x \ge b_0\}$ intersected with the $J \le d-m$ facets $\{n_j^\top x \le h_j\}$ of $\mathcal{X}_0$ from Lemma \ref{lem:local-affine}. One might conclude that the problem is solved, since the region is known explicitly. However, this is not the case because an audit needs to report not a region but two probability-based quantities: (i) the rate $\nu_{\mathrm{local}} = \mathbb{P}_{X \sim \mathcal{N}(x_0, \Sigma)}(X \in S_0) = \int_{S_0} \phi_{x_0, \Sigma}$ where $\phi_{x_0, \Sigma}$ is the $\mathcal{N}(x_0, \Sigma)$ density, and (ii) the violation distribution $Q_{\mathrm{local}}$. The polyhedral description with inequalities supplies neither. On the one hand, knowing the inequalities yields only a membership oracle (testing $x \in S_0$). What makes the present setting tractable is a structural simplification: under Assumption \ref{ass:small-sigma}, the bounding facets of $\mathcal{X}_0$ carry negligible Gaussian mass, so the rate reduces to that of a single half-space, the closed-form tail $\Phi(-d^*_\Sigma)$. On the other hand, recovering the rate instead by scenario generation (drawing from $\mathcal{N}(x_0,\Sigma)$ and keeping the fraction inside $S_0$) costs $\Omega(1/\nu_{\mathrm{local}})$ pipeline solves (Theorem \ref{thm:lower-bound}), since the kept fraction is the rate. Sampling from the polyhedron uniformly instead targets the wrong measure, and reweighting interior draws by $\phi_{x_0, \Sigma}$ to correct this reintroduces exactly the integral (and the normalizing constant $\nu_{\mathrm{local}}$) that should be computed. Our closed-form solution below avoids both routes.

\subsection{Closed-Form Violation Rate.}
\label{subsec:local-rate}

Under $X \sim \mathcal{N}(x_0, \Sigma)$, the scalar $a^\top X$ is univariate Gaussian, $a^\top X \sim \mathcal{N}(a^\top x_0,  a^\top \Sigma a)$. This gives the local violation rate, up to the single-component error of Assumption \ref{ass:small-sigma}. Define the \emph{Mahalanobis signed distance} from $x_0$ to the violation boundary,
\begin{equation}
    \label{eq:signed-dist}
    d^*_\Sigma(x_0)  :=  \frac{b_0 - a^\top x_0}{\sqrt{a^\top \Sigma a}}.
\end{equation}
Here $\sqrt{a^\top \Sigma a} = \|a\|_\Sigma$ is the dual ($\Sigma$-)norm of the covector $a$, equivalently the standard deviation of the scalar $a^\top X$. Thus $d^*_\Sigma(x_0)$ is the Euclidean signed distance from $x_0$ to the violation boundary after the  change of coordinates $x \mapsto \Sigma^{-1/2}x$. Points are measured in the Mahalanobis-norm (using $\Sigma^{-1}$) and hyperplane normals such as $a$ in the dual $\Sigma$-norm, the two being dual under this change of coordinates. Figure \ref{fig:local-geometry} illustrates the terms in our construction.

\begin{figure}
    \centering
    \begin{subfigure}[t]{0.5\textwidth}
    \centering
    \raisebox{0.7cm}{
      \begin{tikzpicture}
      \begin{scope}[scale=1.30]
        \fill[red!12]
          (1.645,0.443) -- (1.461,1.361) -- (-1.385,1.646) -- (-1.554,1.083) -- cycle;
        \draw[thick, fill=blue!5, fill opacity=0.5]
          (-2.206,-1.089) -- (1.123,-1.754) -- (1.906,-0.859)
          -- (1.461,1.361) -- (-1.385,1.646) -- cycle;
        \node[blue!55!black] at (-1.85,-0.95) {\small $\CX_0$};
        \draw[densely dashed, gray!75, rotate around={-140.0:(0,0)}]
          (0,0) ellipse [x radius=1.050, y radius=0.300];
        \draw[densely dashed, gray!55, rotate around={-140.0:(0,0)}]
          (0,0) ellipse [x radius=2.100, y radius=0.600];
        \node[gray!60!black] at (0.95,-0.55) {\small $\mathcal{N}(x_0,\Sigma)$};
        \draw[very thick, red!70!black] (-2.275,1.227) -- (3.609,0.050);
        \node[red!60!black, rotate=-11] at (2.55,0.45) {\small $a^\top x = b_0$};
        \node[red!55!black] at (-1.25,1.35) {\small $S_0$};
        \fill[black] (0,0) circle (1.6pt);
        \node[below left] at (0,0) {\small $x_0$};
        \draw[-{Latex[length=2mm]}, thick, blue!60!black] (0,0) -- (0.667,0.638);
        \fill[blue!60!black] (0.667,0.638) circle (1.0pt);
        \node[blue!55!black, below right, xshift=1pt, yshift=-1pt] at (0.50,0.48) {\small $d^*_\Sigma$};
        \draw[-{Latex[length=1.6mm]}, teal!55!black] (0,0) -- (-0.294,-1.471);
        \node[teal!45!black, right] at (-0.22,-0.95) {\small $\rho_{\min}$};
        \draw[-{Latex[length=2mm]}, red!70!black, thin] (0,0) -- (0.255,1.275);
        \node[red!60!black, right] at (0.20,1.30) {\small $a$};
        \draw[-{Latex[length=2mm]}, blue!60!black, thin] (0,0) -- (0.939,0.899);
        \node[blue!55!black, right] at (0.95,0.92) {\small $v^{*}=\Sigma a/\|\cdot\|$};
      \end{scope}
      \end{tikzpicture}%
    }
    \caption{The optimal basis at $x_0$ fixes a region $\CX_0$ on which the   decision is affine; the violation condition restricts to the half-space $\{a^\top x \ge b_0\}$, and the violation set $S_0$ (shaded) is its intersection with $\CX_0$. The uncertainty $\mathcal{N}(x_0,\Sigma)$ is shown as $1$- and $2$-standard-deviation contours; its covariance makes the most-sensitive direction $v^{*}=\Sigma a/\|\Sigma a\|$ differ from the boundary normal $a$. The margin $\rho_{\min}$ to the nearest facet controls the probability of leaving $\CX_0$ (Assumption \ref{ass:small-sigma}), and $d^{*}_{\Sigma}$ is the Mahalanobis signed distance from $x_0$ to the violation boundary.}
    \end{subfigure}%
    ~~
    \begin{subfigure}[t]{0.5\textwidth}
    \centering
    \resizebox{0.85\textwidth}{!}{%
      \begin{tikzpicture}
        \begin{axis}[
            width=8.0cm, height=6.2cm,
            xmin=-3.6, xmax=3.6, ymin=0, ymax=0.46,
            axis lines=left,
            xtick={0.9}, xticklabels={$d^{*}_{\Sigma}(x_0)$},
            ytick=\empty,
            xlabel={$t=(a^\top x - a^\top x_0)/s$ ~ \tiny(standardized)},
            every axis x label/.style={at={(axis description cs:0.5,-0.1)},anchor=north},
            clip=false,
          ]
          \addplot[domain=-3.6:3.6, samples=200, thick, gray!70!black]
            {1/sqrt(2*pi)*exp(-x^2/2)};
          \addplot[domain=0.9:3.6, samples=120, draw=none, fill=red!18]
            {1/sqrt(2*pi)*exp(-x^2/2)} \closedcycle;
          \draw[very thick, red!70!black] (axis cs:0.9,0) -- (axis cs:0.9,0.40);
          \node[red!60!black, anchor=south] at (axis cs:0.9,0.40) {\small boundary};
          \node[red!55!black, anchor=west] at (axis cs:1.55,0.12)
            {\small $\nu_{\mathrm{local}}=\Phi(-d^{*}_{\Sigma}(x_0))$};
          \draw[-{Latex[length=1.6mm]}, red!55!black]
            (axis cs:1.95,0.10) -- (axis cs:1.45,0.035);
          \fill[black] (axis cs:0,0) circle (1.3pt);
          \node[anchor=north east] at (axis cs:0.225,-0.005) {\small $0$};
        \end{axis}
      \end{tikzpicture}%
    }
      \caption{Projecting onto the boundary normal reduces the rate to a one-dimensional Gaussian tail: the standardized projection $t=(a^\top x - a^\top x_0)/s$ with $s=\sqrt{a^\top\Sigma a}$ is standard normal, and the violation rate $\nu_{\mathrm{local}}=\Phi(-d^{*}_{\Sigma})$ is the shaded mass beyond $d^{*}_{\Sigma}$ (Proposition \ref{prop:local-rate}).}
    \end{subfigure}
    \caption{Geometry of violation characterization.}
    \label{fig:local-geometry}    
\end{figure}

\begin{proposition}[Violation rate]
\label{prop:local-rate}
Under Assumptions \ref{ass:nondeg} and \ref{ass:small-sigma}, the local violation rate $\nu_{\mathrm{local}} := P_{\mathrm{local}}(X \in S)$ satisfies the two-sided bound
\begin{equation}
    \label{eq:local-rate}
    \Bigl| \nu_{\mathrm{local}}  -  \Phi \bigl(-d^*_\Sigma(x_0)\bigr) \Bigr|
     \leq  \varepsilon_0,
\end{equation}
where $\Phi$ is the standard normal CDF. In isotropic case, i.e., $\Sigma = \sigma^2 I$, we have $d^*_\Sigma(x_0) = d^*(x_0)/\sigma$ with $d^*(x_0) = (b_0 - a^\top x_0)/\|a\|$, recovering $\nu_{\mathrm{local}} = \Phi(-d^*(x_0)/\sigma) \pm \varepsilon_0$.
\end{proposition}

\begin{proof}
Let $\nu_{S_0} := P_{\mathrm{local}}(X \in S_0)$ be the single-component violation probability. The events $\{X \in S\}$ and $\{X \in S_0\}$ differ only on $\{X \notin \CX_0\}$: on $\CX_0$ the decision is the affine piece \eqref{eq:pwa}, so $V(x) = \mathbf{1}[a^\top x \ge b_0]$ there, and $S \cap \CX_0 = S_0$. Hence $|\nu_{\mathrm{local}} - \nu_{S_0}| \le P_{\mathrm{local}}(X \notin \CX_0) \le \varepsilon_0$ by Assumption \ref{ass:small-sigma}. Moreover, $S_0 \subseteq \{a^\top x \ge b_0\}$, and the two differ only on $\{X \notin \CX_0\}$, so $\nu_{S_0}$ and $P_{\mathrm{local}}(a^\top X \ge b_0)$ also differ by at most $\varepsilon_0$. Since both estimates of $\nu_{\mathrm{local}}$ are squeezed within the same $\{X \notin \CX_0\}$ event, $|\nu_{\mathrm{local}} - P_{\mathrm{local}}(a^\top X \ge b_0)| \le \varepsilon_0$. Finally, since $a^\top X \sim \mathcal{N}(a^\top x_0, a^\top\Sigma a)$, 
\[
P_{\mathrm{local}}(a^\top X \ge b_0)
= \Phi\left(\frac{a^\top x_0 - b_0}{\sqrt{a^\top\Sigma a}}\right)
= \Phi\bigl(-d^*_\Sigma(x_0)\bigr),
\]
which gives \eqref{eq:local-rate}; see also Figure \ref{fig:local-geometry}. The isotropic case of $\Sigma = \sigma^2 I$ follows from $a^\top(\sigma^2 I)a = \sigma^2\|a\|^2$.
\end{proof}

The Mahalanobis signed distance $d^*_\Sigma(x_0)$ has a direct interpretation: it measures how far $x_0$ is from the violation boundary in units of the perturbation's own standard deviation \emph{along the violation-driving direction} $a$. If $d^*_\Sigma(x_0) > 0$, the point is safe and $\nu_{\mathrm{local}} < 1/2$. However, if $d^*_\Sigma(x_0) \le 0$, then it already violates and $\nu_{\mathrm{local}} \ge 1/2$. This is the appropriate reliability measure under correlated uncertainty: a decision can be far from the boundary in Euclidean terms yet fragile if the perturbation has large variance precisely in the direction $a$. In the case of $\Sigma = \sigma^2 I$, the ratio $d^*(x_0)/\sigma$ implies that all directions carry equal uncertainty.

\begin{remark}[Computational cost]
\label{rem:local-computation}
Computing $\nu_{\mathrm{local}}$ requires: (i) one LP solve at $x_0$ to obtain basis $\mathcal{B}(x_0)$, and hence $C_0$ and $d_0$ via Lemma \ref{lem:local-affine}; (ii) one matrix-vector product for $a = (w^\top C_0)^\top$ and the quadratic form $a^\top \Sigma a$, and the scalar $b_0 = \gamma - w^\top d_0$; (iii) one evaluation of $\Phi$. The total complexity is $\mathcal{O}(\mathrm{LP}(d, m, p) + p^2)$.
\end{remark}

\begin{table}[t]
\centering
\begin{tabular}{ccc}
\hline
$d^*_\Sigma(x_0)$ & $\nu_{\mathrm{local}}$ & $\Omega(1/\nu_{\mathrm{local}})$ samples \\
\hline
$1.0$ & $0.159$ & $\sim 6$ \\
$2.0$ & $0.023$ & $\sim 44$ \\
$3.0$ & $0.0013$ & $\sim 750$ \\
$4.0$ & $3.2\times10^{-5}$ & $\sim 31{,}000$ \\
$5.0$ & $2.9\times10^{-7}$ & $\sim 3.4\times10^{6}$ \\
\hline
\end{tabular}
\caption{Sample complexity of scenario generation for local violation detection as a
function of the Mahalanobis signed distance $d^*_\Sigma(x_0)$ from $x_0$ to the
violation boundary (in the isotropic case $d^*_\Sigma(x_0) = d^*(x_0)/\sigma$). The
closed-form result of Proposition \ref{prop:local-rate} makes scenario generation
unnecessary in all cases.}
\label{tab:local-complexity}
\end{table}

As noted in Section \ref{sec:hardness}, the lower bound of Theorem \ref{thm:lower-bound} applies to $P_{\mathrm{local}}$ without modification, so scenario generation requires $\Omega(1/\nu_{\mathrm{local}})$ perturbed evaluations to detect a local violation rate of $\nu_{\mathrm{local}}$. By Proposition \ref{prop:local-rate}, $\nu_{\mathrm{local}} = \Phi(-d^*_\Sigma(x_0)) \pm \varepsilon_0$, which decreases rapidly as the Mahalanobis margin $d^*_\Sigma(x_0)$ grows--that is, as $x_0$ moves away from the violation boundary relative to the perturbation scale in the direction $a$. Table \ref{tab:local-complexity} illustrates the resulting sample complexity: at a Mahalanobis distance of $5.0$, scenario generation would require over three million perturbed evaluations to detect a rate that Proposition \ref{prop:local-rate} returns in closed form from a single evaluation of $\Phi$. 

\subsection{Direct Sampler for the Local Violation Distribution.}
\label{subsec:local-sampler}

The single-component target $\mathcal{N}(x_0, \Sigma)(\cdot \mid a^\top X \ge b_0)$ is a Gaussian truncated to a half-space. It admits an exact direct sampler built from the Gaussian conditional decomposition along the projection $T := a^\top X$. Under $X \sim \mathcal{N}(x_0, \Sigma)$
\begin{equation}
    \label{eq:T-marginal}
    T = a^\top X \sim \mathcal{N}\bigl(a^\top x_0, a^\top \Sigma a\bigr),
\end{equation}
and the conditional distribution of $X$ given $T = t$ is Gaussian with mean and covariance
\begin{equation}
    \label{eq:gauss-cond}
    \mu_{X\mid T}(t) = x_0 + \frac{\Sigma a}{a^\top \Sigma a}(t - a^\top x_0),
    \qquad
    \Sigma_{X\mid T} = \Sigma - \frac{\Sigma a a^\top \Sigma}{a^\top \Sigma a},
\end{equation}
the latter being the Schur complement of $a^\top\Sigma a$ in $\Sigma$. The violation condition $a^\top X \ge b_0$ is the event $T \ge b_0$, equivalently $(T - a^\top x_0)/\sqrt{a^\top\Sigma a} \ge -d^*_\Sigma(x_0)$, and constrains only $T$.

We state the sampler's accuracy in total-variation distance. For two distributions $\mu, \nu$ on the same space, the total-variation distance $\mathrm{TV}(\mu,\nu) = \sup_A |\mu(A) - \nu(A)|$ is the largest difference in the probability they assign to any event $A$. A bound $\mathrm{TV}(\mu,\nu) \le \eta$ guarantees that the two distributions assign probabilities differing by at most $\eta$ to every event, and expectations of any function bounded by one differing by at most $2\eta$.

\begin{proposition}[Direct sampler]
\label{prop:local-sampler}
Let $s := \sqrt{a^\top \Sigma a}$. Under Assumption \ref{ass:nondeg} and Assumption \ref{ass:small-sigma}, the following procedure generates exact i.i.d. samples from the single-component target $\mathcal{N}(x_0, \Sigma)(\cdot \mid a^\top X \ge b_0)$:
\begin{enumerate}
    \item[(i)] Draw $U \sim \mathcal{U}(0, 1)$ and set
    \begin{equation}
        \label{eq:truncated-normal}
        T = a^\top x_0 + s\Phi^{-1}\left(
        \Phi\bigl(d^*_\Sigma(x_0)\bigr)
        + U\bigl(1 - \Phi\bigl(d^*_\Sigma(x_0)\bigr)\bigr)
        \right).
    \end{equation}
    That is, $T$ from $\mathcal{N}(a^\top x_0, s^2)$ truncated to $[b_0, \infty)$. Note that $(b_0 - a^\top x_0)/s = d^*_\Sigma(x_0)$, so the truncation occupies the upper tail of probability $1 - \Phi(d^*_\Sigma(x_0)) = \Phi(-d^*_\Sigma(x_0)) = \nu_{\mathrm{local}}$ (up to $\varepsilon_0$).
    \item[(ii)] Draw $X$ from the Gaussian conditional $\mathcal{N}\bigl(\mu_{X\mid T}(T), \Sigma_{X\mid T}\bigr)$ of \eqref{eq:gauss-cond}.
    \item[(iii)] Return $X$.
\end{enumerate}
The samples produced have distribution within total-variation distance $\mathcal{O}(\varepsilon_0/\nu_{\mathrm{local}})$ of $Q_{\mathrm{local}}$. After an $\mathcal{O}(p^2)$ precomputation (of $\Sigma a$, $s$, and a factor of $\Sigma_{X\mid T}$), each sample costs $\mathcal{O}(p^2)$ operations.
\end{proposition}

\begin{proof}
Step (i) draws $T$ from the correct truncated marginal of \eqref{eq:T-marginal} by the probability integral transform, using $P(T \ge b_0) = \Phi(-d^*_\Sigma(x_0))$. Step (ii) draws $X$ from the exact Gaussian conditional \eqref{eq:gauss-cond}. Since the truncation constrains only $T = a^\top X$ and $X \mid T$ is unaffected by it, the pair is distributed as $\mathcal{N}(x_0, \Sigma)$ conditioned on $a^\top X \ge b_0$, i.e., the single-component target. Finally, $Q_{\mathrm{local}}$ is the same Gaussian conditioned on $X \in S$; the two conditioning events $\{X \in S\}$ and $\{a^\top X \ge b_0\}$ differ only on $\{X \notin \CX_0\}$, which has  probability at most $\varepsilon_0$ under Assumption \ref{ass:small-sigma}. Conditioning on two events whose symmetric difference has probability $\le \varepsilon_0$, each with probability $\ge \nu_{\mathrm{local}}$, changes the normalized distribution by at most $\mathcal{O}(\varepsilon_0/\nu_{\mathrm{local}})$ in total variation: the unnormalized measures differ by $\le \varepsilon_0$ and the normalizers differ by $\le \varepsilon_0$, so the ratio bound is $2\varepsilon_0/\nu_{\mathrm{local}}$ in the worst case. Hence the sampler's output is within $\mathcal{O}(\varepsilon_0/\nu_{\mathrm{local}})$ total variation of $Q_{\mathrm{local}}$.
\end{proof}

The two error modes differ in character. On the one hand, the rate $\nu_{\mathrm{local}}$ incurs only an \emph{additive} error $\mathcal{O}(\varepsilon_0)$ (Proposition \ref{prop:local-rate}), which is benign no matter how small $\nu_{\mathrm{local}}$ is because the same geometry that makes $\nu_{\text{local}}$ small makes $\varepsilon_0$ smaller still. On the other hand, whereas the conditional distribution incurs a \emph{relative} error $\mathcal{O}(\varepsilon_0/\nu_{\mathrm{local}})$. Read naively, this ratio appears to diverge in the rare-violation regime ($\nu_{\mathrm{local}} \to 0$). It does not, because $\varepsilon_0$ and $\nu_{\mathrm{local}}$ are tail probabilities of the \emph{same} Gaussian: $\varepsilon_0 \le J\Phi(-\rho_{\min}(x_0))$ is the mass beyond the nearest basis facet (at Mahalanobis distance $\rho_{\min}$) and $\nu_{\mathrm{local}} = \Phi(-d^*_\Sigma(x_0))$ the mass beyond the violation boundary (at distance $d^*_\Sigma$). Since $t \mapsto \Phi(-t)e^{t^2/2}$ is decreasing on $(0,\infty)$\footnote{The Gaussian Mills-ratio bound $\Phi(-t) < \phi(t)/t$.}, for $\rho_{\min} > d^*_\Sigma$ the ratio obeys
\begin{equation}
    \label{eq:tv-refined}
    \frac{\varepsilon_0}{\nu_{\mathrm{local}}}
    \le J\frac{\Phi(-\rho_{\min})}{\Phi(-d^*_\Sigma)}
    \le J\exp\!\Bigl(-\tfrac12\bigl(\rho_{\min}^2 - d^{*2}_\Sigma\bigr)\Bigr),
\end{equation}
governed by the \emph{gap between the squared distances}, not by the magnitude of $\nu_{\mathrm{local}}$. A rare violation ($d^*_\Sigma$ large, $\nu_{\mathrm{local}}$ small) is therefore sampled accurately provided the nearest facet is even farther; in particular, a target total variation $\delta$ is guaranteed once
\begin{equation}
    \label{eq:margin-condition}
    \rho_{\min}(x_0)^2 - d^*_\Sigma(x_0)^2 \ge 2\log(J/\delta),
\end{equation}
a fixed additive margin in squared distance that is independent of how small $\nu_{\mathrm{local}}$ is. The bound degrades only when $\rho_{\min} \lesssim d^*_\Sigma$, i.e., when a basis change intrudes inside the violation boundary, the regime where the single-region model itself fails. Equivalently, the basis region must extend toward its facets at least as far as the violation boundary lies in the direction $a$, which holds whenever the deployed vertex is stable.

\begin{remark}[Decomposition]
\label{rem:whitening}
For $Y = \Sigma^{-1/2}(X - x_0) \sim \mathcal{N}(0, I)$, the violation condition becomes $\tilde a^\top Y \ge \tilde b$ with $\tilde a = \Sigma^{1/2} a$, a half-space through a \emph{standard} Gaussian. One may then run the elementary sampler (truncate the component along $\tilde a/\|\tilde a\|$, draw the orthogonal complement as standard normal) in $Y$-space and map back via $X = x_0 + \Sigma^{1/2} Y$. This recovers \eqref{eq:gauss-cond} and is convenient when a factor $\Sigma^{1/2}$ is already available.
\end{remark}

The rate (Proposition \ref{prop:local-rate}) and the exact sampler (Proposition \ref{prop:local-sampler}) assemble into a single routine as given in Algorithm \ref{alg:local-certificate}. Stage 1 returns the certificate from a single LP solve, and Stage 2 draws exact samples for any statistic not in closed form. Using Algorithm \ref{alg:local-certificate}, we plot in Figure \ref{fig:complexity} the i.i.d. draws scenario generation requires as a function of $d^*_\Sigma$, for three tasks of increasing difficulty: detecting one violation ($\sim 1/\nu_{\mathrm{local}}$), estimating the rate to $30\%$ relative standard error ($\sim 1/(\nu_{\mathrm{local}}\epsilon^2)$), and collecting $30$ violating samples for distributional summaries ($\sim 30/\nu_{\mathrm{local}}$), following the $\Omega(1/\nu_{\mathrm{local}})$ bound of Theorem \ref{thm:lower-bound}. At $d^*_\Sigma = 5$ even mere detection exceeds three million pipeline evaluations, each an optimization solve, whereas the closed-form result of Proposition \ref{prop:local-rate} returns the rate from a single LP solve. Figure \ref{fig:efficiency} shows the operational counterpart: to generate $10^4$ valid violating scenarios, the rejection sampler needs $\sim 10^4/\nu_{\mathrm{local}}$ raw draws (exceeding $10^9$ at $d^*_\Sigma = 4.5$), while the direct sampler needs exactly $10^4$.

\begin{algorithm}[t]
\caption{Violation Certificate and Exact Sampler}
\label{alg:local-certificate}
\begin{algorithmic}[1]
\Require pipeline data $(B, b, A, c)$, violation $(w, \gamma)$, reference input
  $x_0$, covariance $\Sigma \succ 0$, sample count $n \ge 0$
\Statex \textbf{Stage 1: Certificate}
\State Solve the LP \eqref{eq:lp} at $\theta(x_0) = B x_0 + b$; let
  $\mathcal{B}_0$ be the optimal basis \Comment{Lemma \ref{lem:local-affine}}
\State $C_0 \gets A_{\mathcal{B}_0}^{-1} B_{\mathcal{B}_0}$,\quad
  $d_0 \gets A_{\mathcal{B}_0}^{-1} b_{\mathcal{B}_0}$
\State $a \gets (w^\top C_0)^\top$,\quad
  $b_0 \gets \gamma - w^\top d_0$,\quad
  $s \gets \sqrt{a^\top \Sigma a}$
\State $d^*_\Sigma \gets (b_0 - a^\top x_0)/s$;\quad
  $\nu_{\mathrm{local}} \gets \Phi(-d^*_\Sigma)$;\quad
  $v^* \gets \Sigma a / s$
\State \textbf{output} certificate $(d^*_\Sigma,\ \nu_{\mathrm{local}},\ v^*)$

\Statex \textbf{Stage 2: Exact Sampler} \Comment{when $n \geq 1$}
\State Precompute $\Sigma a$, the slope $m_T \gets \Sigma a / s^2$, and a factor $L$
  with $L L^\top = \Sigma - \Sigma a a^\top \Sigma / s^2$
  \Comment{eq. \eqref{eq:gauss-cond}}
\For{$i = 1$ to $n$}
  \State draw $U \sim \mathcal{U}(0,1)$;\quad
    $T \gets a^\top x_0 + s\Phi^{-1}\bigl(\Phi(d^*_\Sigma) + U(1 - \Phi(d^*_\Sigma))\bigr)$
    \Comment{eq. \eqref{eq:truncated-normal}}
  \State draw $Z \sim \mathcal{N}(0, I)$;\quad
    $X_i \gets x_0 + (T - a^\top x_0) m_T + L Z$
\EndFor
\State \textbf{return} samples $\{X_i\}_{i=1}^n$
\end{algorithmic}
\end{algorithm}

\begin{figure}[t]
\centering
\begin{subfigure}[b]{0.48\textwidth}
  \centering
  \includegraphics[width=\textwidth]{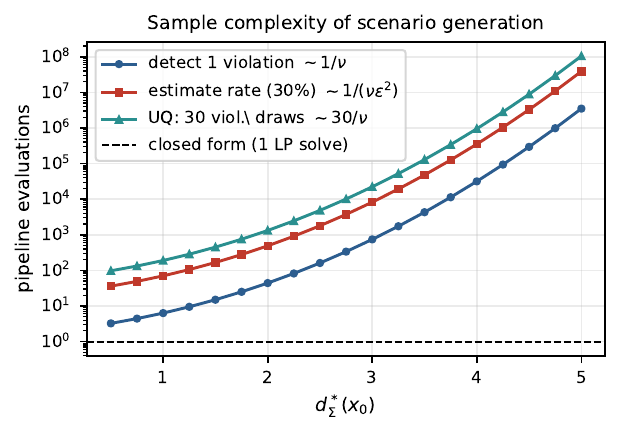}
  \caption{Draws required for detection, rate estimation, and distributional summaries vs. the flat closed-form cost.}
  \label{fig:complexity}
\end{subfigure}
\hfill
\begin{subfigure}[b]{0.48\textwidth}
  \centering
  \includegraphics[width=\textwidth]{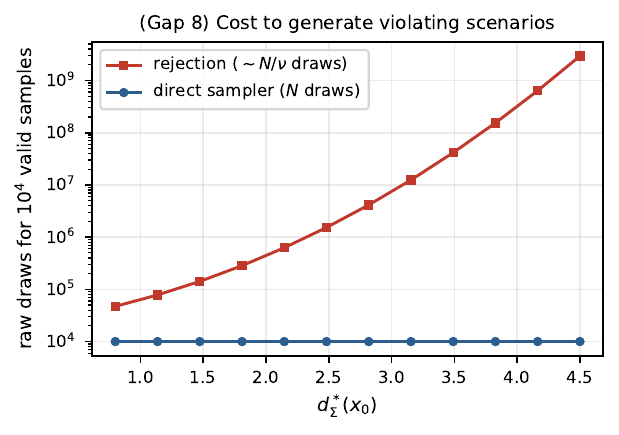}
  \caption{Raw draws to obtain $10^4$ valid violating samples: rejection ($\sim 10^4/\nu$) vs. the direct sampler ($10^4$).}
  \label{fig:efficiency}
\end{subfigure}
\caption{The cost of sampling-based scenario generation grows like $1/\nu_{\mathrm{local}}$, both information-theoretically (a) and operationally (b), while the closed-form certificate and the direct sampler are unaffected.}
\label{fig:complexity-efficiency}
\end{figure}

\subsection{The Violation Certificate: Statistics and Interpretation.}
\label{subsec:local-audit}

The closed-form rate and the conditional distribution of Proposition \ref{prop:local-sampler} yield, the statistics a practitioner would normally estimate by accumulating violating draws. In the subsequent result, we write $\alpha := d^*_\Sigma(x_0)$, $s := \sqrt{a^\top \Sigma a}$, and let $\lambda := \phi(\alpha)/\bigl(1 - \Phi(\alpha)\bigr)$ be the inverse Mills ratio.

\begin{corollary}[Closed-form violation statistics]
\label{cor:statistics}
Under the assumptions of Proposition \ref{prop:local-sampler}, the following hold in closed form, where all conditioning is on the violation event $\{X \in S\} = \{a^\top X \ge b_0\}$:
\begin{align}
    \mathbb{E}[X \mid X \in S]
      &= x_0 + \frac{\Sigma a}{s}\lambda, \label{eq:cond-mean}\\[2pt]
    \mathrm{Cov}[X \mid X \in S]
      &= \Sigma - \frac{\Sigma a a^\top \Sigma}{s^2}\lambda(\lambda - \alpha),
      \label{eq:cond-cov}\\[2pt]
    \mathbb{E}[w^\top Z_0(X) - \gamma \mid X \in S]
      &= s(\lambda - \alpha), \label{eq:severity}\\[2pt]
    \frac{\partial \nu_{\mathrm{local}}}{\partial x_0}
      &= \phi(\alpha)\frac{a}{s}. \label{eq:sensitivity}
\end{align}
\end{corollary}

The proof, a specialization of the standard truncated-Gaussian moments to the half-space $S$, is given in Appendix \ref{app:corollary}. Equation \eqref{eq:cond-mean} is the \emph{typical violating input}; \eqref{eq:cond-cov} its spread; \eqref{eq:severity} the \emph{expected violation severity}, using the identity $w^\top Z_0(x) - \gamma = a^\top x - b_0$ on $\mathcal{X}_0$, with tail quantiles of the severity available from the univariate truncated normal of $a^\top X \mid X \in S$; and \eqref{eq:sensitivity} the \emph{sensitivity} of the rate to the deployed input. Each of these would otherwise require accumulating $\Omega(k/\nu_{\mathrm{local}})$ violating evaluations by scenario generation to estimate from $k$ samples. Here, they are instead just formulas.

Together, the Mahalanobis signed distance $d^*_\Sigma(x_0)$ and the rate $\nu_{\mathrm{local}}$ constitute an \emph{individual violation certificate} for the decision at $x_0$: a quantitative, auditable statement about the reliability of the specific decision as deployed, computable from a single LP solve. A large $d^*_\Sigma(x_0)$ means the decision is deep inside the safe region relative to the prescribed uncertainty and rarely violates; a small or negative value means it is fragile and warrants scrutiny or pipeline revision. The certificate is also explanatory. Violations are driven by displacement along $v^* = \Sigma a/\sqrt{a^\top \Sigma a}$, the direction that maximizes the violation probability per unit Mahalanobis perturbation; its largest-magnitude components identify the input features whose perturbation poses the greatest local risk, and the sensitivity \eqref{eq:sensitivity} makes this attribution quantitative. In the isotropic case $\Sigma = \sigma^2 I$, $v^*$ reduces to the unit vector $\hat a = a/\|a\|$ along the gradient $w^\top C_0$ of the violation criterion. The typical violating input \eqref{eq:cond-mean} and the expected severity \eqref{eq:severity} complete the picture: not only how often the decision fails, but what a failure looks like and how bad it is.

\section{An Application Pipeline: Economic Dispatch under an Emissions Cap.}
\label{sec:dispatch}

We illustrate the certificate on a optimization problem in which the violation half-space is not posited but \emph{emerges} from the cost structure and the active basis. This is the sense in which it is computed end-to-end from the pipeline rather than configured. Forecast-driven economic dispatch is a canonical predict-then-optimize application \citep{chen2022feature}. Although the unit-commitment problem is mixed integer, the continuous relaxation below maintains its structure: forecasts feed predicted demand, a least-cost dispatch is solved, and a regulatory cap on the result defines compliance. 


A system operator meets forecast electricity demand at least cost by dispatching five generators, choosing outputs $z \in \mathbb{R}^5$ to solve the economic-dispatch LP:
\begin{equation}
    \label{eq:dispatch-lp}
    \min_{z} ~ c^\top z
    \quad\text{s.t.}\quad
    \mathbf{1}^\top z = \theta_D(x),  ~ 0 \le z \le \bar z,
\end{equation}
where $c$ are marginal costs, $\bar z$ capacities, and the demand requirement $\theta_D(x) = B x + b$ is predicted from two forecast features $x = (x_{(1)}, x_{(2)})$. These are a system-load index and a renewable-availability index. The generators and their parameters are given in Table \ref{tab:generators}. A regulatory emissions cap requires the dispatched mix to satisfy $w^\top z \le \gamma$, where $w$ are the per-unit emission factors (column 4 of Table \ref{tab:generators}) and $\gamma$ is the cap; a violation is an emissions overshoot $w^\top z^\star(x) > \gamma$.

\begin{table}[t]
\centering
\small
\begin{tabular}{lccc}
\hline\hline
Generator & Cost (\$/MWh) & Capacity (MW) & Emissions (tCO$_2$/MWh) \\
\hline
G1 & 20 & 200 & 0.95 \\
G2 & 25 & 180 & 0.85 \\
G5 & 30 & 260 & 0.55 \\
G3 & 38 & 150 & 0.35 \\
G4 & 45 & 120 & 0.05 \\
\hline\hline
\end{tabular}
\caption{Generators for the dispatch example, listed in merit (cost) order. At the deployed forecast the cheapest units run at capacity and the mid-cost unit G5 is marginal; the cleaner but costlier units G3, G4 stay off.}
\label{tab:generators}
\end{table}

\paragraph{Certificate.}
At the deployed forecast $x_0 = (1.0, 0.5)$ the predicted demand is $\theta_D(x_0) = 500$ MW. Solving \eqref{eq:dispatch-lp} once returns the merit-order dispatch: the two cheapest units run at capacity (G1 at $200$, G2 at $180$ MW), the two cleaner but costlier units stay off (G3, G4 at $0$), and the mid-cost unit G5 is the marginal generator at $120$ MW (Figure \ref{fig:dispatch}(a)). This active basis fixes the affine piece $z(x) = C_0 x + d_0$: with G1, G2 capped and G3, G4 off, only G5 responds to demand, so $z_{\mathrm{G5}}(x) = \theta_D(x) - 380$. The violation functional inherits this structure. Emissions are $w^\top z(x) = 0.95(200) + 0.85(180) + 0.55 z_{\mathrm{G5}}(x)$, affine in $x$ through G5, so the violation half-space normal is \emph{derived} from the marginal unit's emission factor and the demand sensitivity, \begin{equation}
    \label{eq:dispatch-a}
    a = (w^\top C_0)^\top = 0.55 B^\top = (66,  -22)^\top,
\end{equation}
not chosen. With the forecast covariance 
\[
\Sigma = \left[
\begin{array}{cc}
    0.025 & 0.008  \\
     0.008 & 0.02 
\end{array}
\right]
\]
and an emissions cap $\gamma$ placing the deployed dispatch at emissions $409$ tCO$_2$ against a cap of $428.5$ tCO$_2$ (a margin of $19.5$ tCO$_2$), the certificate reads 
\begin{equation}
    \label{eq:dispatch-cert}
    d^*_\Sigma(x_0) = 2.00, \qquad \nu_{\mathrm{local}} = \Phi(-2.00) = 0.023,
\end{equation}
a $2.3\%$ chance that forecast error drives the least-cost dispatch over the emissions cap. The deployed vertex is stable, i.e., the nearest basis-change margin is $\rho_{\min} = 6.8 \gg d^*_\Sigma$, so Assumption \ref{ass:small-sigma} holds comfortably and the single-region closed form is accurate.

\paragraph{What the certificate reports.} Beyond the rate, Corollary \ref{cor:statistics} returns an interpretable audit (Figure \ref{fig:dispatch}(b)). When the cap is breached, the typical triggering forecast shift is concentrated almost entirely in the load feature ($+0.36$) rather than the renewable feature ($+0.02$): violations are demand-driven, not renewable-driven. The conditional severity is $3.6$ tCO$_2$ of expected overshoot, and the rate sensitivity attributes roughly three times more influence to the load feature than the renewable feature, identifying where tighter forecasting would most reduce compliance risk. All of this follows from the single dispatch solve at $x_0$, with no scenario generation.

\paragraph{The cost avoided.}
Every quantity above came from that one solve. Estimating the typical violating forecast by scenario generation would instead require accumulating violating draws by rejection, at a cost of $1/\nu_{\mathrm{local}} \approx 44$ dispatch solves each. Pinning the conditional mean to about $1\%$ relative accuracy here needs on the order of $10^4$ violating draws, where the count is set by the renewable feature $\mathbb{E}[X_{(2)} \mid X \in S] = 0.52$, whose conditional standard deviation is largest and whose relative precision is therefore worst. This is roughly $3 \times 10^4$ dispatch solves, with coarser targets ($10\%$ accuracy, or the rate itself) costing $\sim 10^3$ solves. The conditional mean, severity, and attribution are instead returned in closed form, exactly and at every feature scale. These explanation outputs (the attribution direction and the typical-failure forecast) are the per-decision evidence that transparency and human-oversight obligations for high-risk AI systems call for \citep[EU AI Act, Arts.~13--14;][]{EUAIA2024}, supporting rather than supplanting the operator.

\paragraph{Why the polyhedron alone would not do.}
It is worth discussing on this instance why the explicit violation set is not itself the certificate. The basis region $\mathcal{X}_0$ here is bounded by $J = 2$ facets (the demand levels at which the marginal generator G5
hits zero or its capacity) and $S_0$ is their intersection with the emissions half-space, a complete polyhedral description. Yet the audit needs the Gaussian \emph{mass} of that polyhedron, not its inequalities. Two facts make the point. First, the nearest facet sits $\rho_{\min} = 6.8$ standard deviations from $x_0$ in the $\Sigma$-metric, so the mass leaking past the facets is $J\exp(-\tfrac12\rho_{\min}^2) \approx 2 \times 10^{-10}$: the Gaussian measure of the full polyhedron $S_0$ and of the single emissions half-space agree to all reported digits ($0.0227$ either way), the collapse that Proposition \ref{prop:local-rate} makes exact up to $\varepsilon_0$ and that yields $\nu_{\mathrm{local}} = \Phi(-2.00)$ in closed form. Knowing the facets does not reveal this, the probabilistic argument does. Second, the distribution cannot be recovered by sampling the polyhedron under a convenient law: drawing \emph{uniformly} from $S_0$ places the typical violating forecast at a load level of $1.70$, wrongly locating the true conditional mean $1.36$ by more than $0.3$, because the violation distribution is the Gaussian restricted to $S_0$, which concentrates near the boundary, not the uniform law. Sampling $S_0$ correctly means sampling that restricted Gaussian, which is precisely $Q_{\mathrm{local}}$. The direct sampler (Proposition \ref{prop:local-sampler}) draws from it exactly, and Corollary \ref{cor:statistics} returns its moments without sampling at all.

\begin{figure}[t]
\centering
\includegraphics[width=\textwidth]{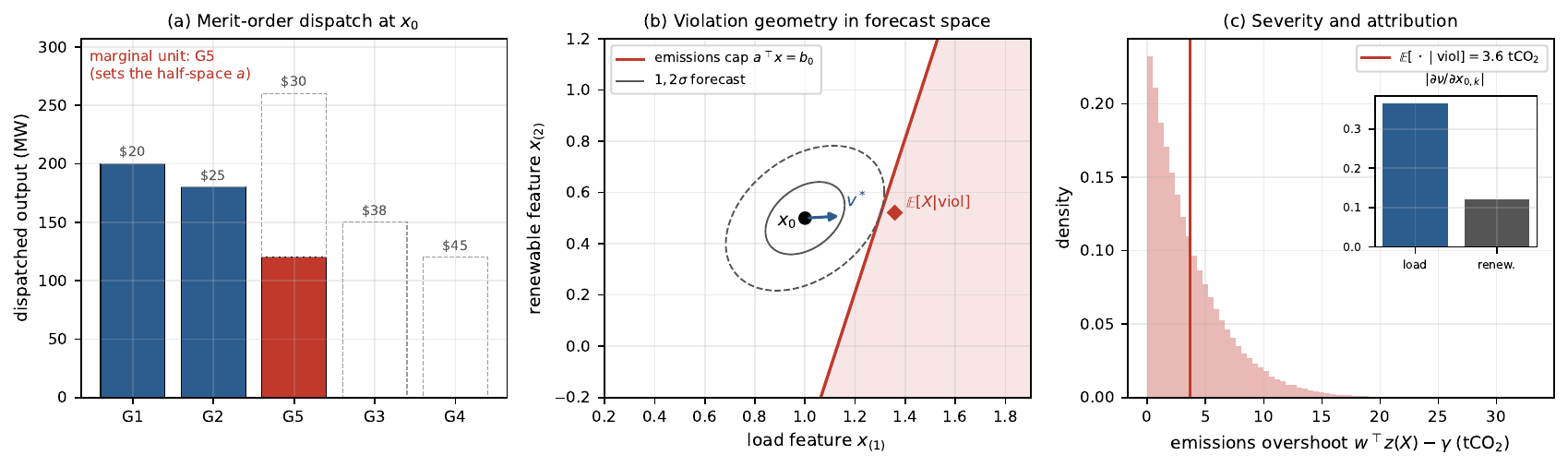}
\caption{Economic-dispatch example. (a) The merit-order dispatch at the deployed forecast: units are ordered by cost, the cheapest (G1, G2) run at capacity (solid), the cleaner costlier units (G3, G4) stay off, and the mid-cost unit G5 is marginal (red); that is G5's emission factor that, through the active basis, sets the violation half-space normal $a$ in \eqref{eq:dispatch-a}. Dashed outlines show unused capacity. (b) The induced violation geometry in forecast space: the deployed forecast $x_0$ under its uncertainty $\mathcal{N}(x_0, \Sigma)$ ($1$- and $2\sigma$ contours), the emissions-cap boundary $a^\top x = b_0$ (shaded violation region), the most-sensitive direction $v^*$, and the conditional-mean violating forecast. (c) The closed-form severity distribution (emissions overshoot $w^\top z(X) - \gamma$ given a violation, mean $3.6$ tCO$_2$ marked) with the feature-attribution sensitivities $|\partial\nu_{\mathrm{local}}/\partial x_{0,k}|$ inset, showing the load feature dominating.}
\label{fig:dispatch}
\end{figure}

\section{Discussion.}
\label{sec:discussion}

Our work shows that exact certificates can be obtained when the pipeline consists entirely of linear stages. In this section, we discuss the difficulties of extending these results beyond the setting we studied. Specifically, we examine the role of linearity in enabling exact certification, clarify what ``larger scale'' does and does not imply for local certificates, and explain why extending certification beyond a single basis region remains an open challenge that requires further investigation

\paragraph{What linearity brings.} The whole local characterization rests on Lemma \ref{lem:local-affine}: on the basis region $\mathcal{X}_0$ the optimal solution is the affine map $z(x) = C_0 x + d_0$, because the predictor is affine and only the right-hand side of the linear program varies with $x$. Two departures break this. If the predictor is nonlinear, the map $x \mapsto \theta(x)$ is no longer affine, so the pre-image $\mathcal{X}_0$ of a basis cone is no longer a polyhedron and the induced law of $a^\top X$ is no longer Gaussian. If the predicted parameter enters the \emph{cost} rather than only the right-hand side, the optimal basis itself varies nonlinearly with $\theta$ and the basic solution is no longer affine. In either case the closed-form rate $\Phi(-d^*_\Sigma)$ and the exact sampler fail together. The linear theory is therefore sharply delimited rather than incidentally restricted: it is exactly the regime in which the optimal-solution map is affine and the violation set is a half-space. The violation set $S = \{x : w^\top z(x) \ge \gamma\}$ remains well defined for any pipeline; what is lost is its closed form. One natural route for the nonlinear case is to sample a tempered relaxation of $S$ that recovers the present closed form in its zero-temperature limit; analyzing such a sampler is left to future work.

\paragraph{What ``larger scale'' means.} A natural question is whether the certificate scales. The answer requires separating two notions of size that behave very differently. The \emph{decision-space} size (the number of variables and constraints of the linear program) is essentially free: the certificate geometry lives entirely in input space, and the optimization enters only through the single solve that fixes the active basis and hence the half-space normal $a = (w^\top C_0)^\top$. The economic-dispatch pipeline of Section \ref{sec:dispatch} illustrates this directly: it carries five decision variables and a full set of capacity constraints, yet its certificate is a single one-dimensional rate $\Phi(-d^*_\Sigma)$ obtained from one LP solve. More generally, a model with hundreds of constraints yields the same one-dimensional rate at the cost of one such solve. The \emph{input-space} dimension $p$ (the number of uncertain forecast features) is the size that the formulas of Corollary \ref{cor:statistics} actually depend on, and there the cost is the $O(p^2)$ of a single Gaussian conditioning, with no dependence on the violation rate. Thus ``scaling up'' the optimization does not stress the certificate at all; only the dimension of the uncertainty does, and that dependence is polynomial and sampling-free. This is why a recognizable, realistically sized pipeline and a low-dimensional uncertainty model are not in tension: the former exercises the decision space, the latter keeps the certificate geometry interpretable.

\paragraph{The multi-region frontier.}
The single binding assumption is that the deployed vertex is stable: the perturbation reaches the violation boundary before it crosses into a neighboring optimal-basis region (Assumption \ref{ass:small-sigma}, $\rho_{\min} \gtrsim d^*_\Sigma$). When this fails, i.e., when a basis-change facet of $\mathcal{X}_0$ is nearer than the violation boundary, the violation set is no longer a single half-space but a union over basis regions, and the closed form ceases to be valid. This is not a limitation that a tighter argument removes; it is a genuine change in the object. A perturbation large enough to reach the violation boundary may first re-solve the linear program at a different vertex, where the committed decision is governed by a \emph{different} affine piece and hence a different violation half-space. Characterizing the rate then requires three things the single-region analysis avoids: enumerating the neighboring bases that the perturbation can reach, computing each region's half-space and its Gaussian mass, and combining them into a mixture, all while the number of reachable regions can grow combinatorially with the dimension of the model (up to $\binom{d}{m}$ in the worst case). The regions also need not align with the perturbation geometry, so the clean Mahalanobis picture of a single signed distance is replaced by a more intricate arrangement. The dispatch pipeline of Section \ref{sec:dispatch} makes the failure concrete. Suppose the marginal generator G5 had only narrow headroom above its deployed output (a tight capacity rather than the comfortable margin of Table \ref{tab:generators}). A demand surprise would then drive G5 to its capacity \emph{before} emissions reached the cap, at which point the next unit in the merit order, the cleaner G3 (emission factor $0.35$ against G5's $0.55$), becomes marginal. Past that basis change, additional demand is met by a lower-emitting generator, so emissions rise \emph{more slowly} than the single half-space, which extrapolates G5's emission rate indefinitely, predicts. The single-region certificate then \emph{overstates} the violation rate: in such an instance the closed form reports $\nu_{\mathrm{local}} = 0.023$ while the true rate is $0.008$, an overstatement of nearly threefold, precisely because $\rho_{\min} < d^*_\Sigma$ and the merit order flips inside the violation boundary. This multi-region regime is where the local certificate meets the global problem, and developing it (region enumeration, per-region certificates, and mixture aggregation with complexity governed by the model's combinatorial structure rather than by the violation rate) is the central direction the present result opens.

\section{Conclusion.}
\label{sec:conclusion}

We provide a closed-form characterization of the local violation behavior of a linear predict-then-optimize pipeline at a fixed deployed input. Our analysis begins with an information-theoretic limitation: detecting violations through scenario generation requires a number of samples that scales inversely with the violation rate. As a result, when violations are rare (the regime that is often of greatest practical interest) sampling-based detection becomes prohibitively expensive.

We then show that, under an affine predictor, a linear optimization problem, and a linear violation criterion, the local violation rate induced by Gaussian perturbations around a fixed input can be expressed in closed form, with an explicitly bounded approximation error. Building on this result, we derive an exact sampler for the local violation distribution. The sampler operates by first drawing from a truncated univariate Gaussian distribution and then sampling from the corresponding Gaussian conditional distribution. Our framework also yields an individual violation certificate consisting of the Mahalanobis signed distance and the associated violation rate. In addition, the most sensitive direction provides feature-level attribution, enabling interpretation of which input features contribute most strongly to risk. These quantities can be computed from a single linear-program solve and require no sampling procedures.

As a per-decision, auditable robustness assessment, the resulting certificate directly addresses the accuracy, robustness, and risk-management considerations that are increasingly emphasized in emerging regulations for high-risk AI systems, many of which will take effect in the coming years.

The local certificate developed here is the component most directly relevant to an individual deployed decision. It requires no sampling and introduces only a provably bounded linearization error. It also serves as the tractable foundation of a broader problem. The global violation distribution over the pipeline's full operating distribution can be viewed as a mixture across the model's optimal-basis regions, whose number may grow combinatorially. Constructing a global certificate therefore requires combining region enumeration, per-region certificates, and mixture aggregation, with computational complexity determined by the model's combinatorial structure rather than the violation rate. Extending this framework to convex, mixed-integer, and nonlinear pipelines is a key direction for future work.

\clearpage

\bibliographystyle{plainnat}
\bibliography{certify_refs}

\clearpage

\appendix

\vspace{3mm}
\begin{center}
\textbf{\large APPENDIX}    
\end{center}
\vspace{5mm}

\noindent\textbf{Proof of Corollary \ref{cor:statistics}.}
\label{app:corollary}
Let $T := a^\top X$, so $T \sim \mathcal{N}(a^\top x_0,  s^2)$ with $s^2 = a^\top \Sigma a$, and the violation event is $\{X \in S\} = \{T \ge b_0\}$. Write the standardized threshold $\alpha = (b_0 - a^\top x_0)/s = d^*_\Sigma(x_0)$ and the inverse Mills ratio $\lambda = \phi(\alpha)/(1 - \Phi(\alpha))$.

\emph{Truncated moments of $T$.} For a standard normal $Z$ conditioned on $Z \ge \alpha$, the classical truncated-normal moments are $\mathbb{E}[Z \mid Z \ge \alpha] = \lambda$ and $\mathrm{Var}[Z \mid Z \ge \alpha] = 1 - \lambda(\lambda - \alpha)$. Writing $T = a^\top x_0 + sZ$ gives
\begin{equation}
\label{eq:T-moments}
    \mathbb{E}[T \mid T \ge b_0] = a^\top x_0 + s\lambda, \qquad
    \mathrm{Var}[T \mid T \ge b_0] = s^2\bigl(1 - \lambda(\lambda - \alpha)\bigr).
\end{equation}

\emph{Conditional mean and covariance of $X$ \eqref{eq:cond-mean}--\eqref{eq:cond-cov}.} Decompose $X$ along $T$ using the Gaussian conditional of Proposition \ref{prop:local-sampler}: $X = x_0 + (T - a^\top x_0) \Sigma a/s^2 + W$, where $W \sim \mathcal{N}(0,  \Sigma - \Sigma a a^\top\Sigma/s^2)$ is independent of $T$. Conditioning on $\{T \ge b_0\}$ affects only the $T$-component, since $W \perp T$. Taking expectations and using \eqref{eq:T-moments}, we have
\[
    \mathbb{E}[X \mid X \in S]
    = x_0 + \bigl(\mathbb{E}[T \mid T \ge b_0] - a^\top x_0\bigr)\frac{\Sigma a}{s^2}
    = x_0 + \frac{\Sigma a}{s}\lambda,
\]
which is \eqref{eq:cond-mean}. For the covariance, independence of $W$ and $T$ gives
\[
    \mathrm{Cov}[X \mid X \in S]
    = \frac{\Sigma a  a^\top\Sigma}{s^4} \mathrm{Var}[T \mid T \ge b_0]
      + \Bigl(\Sigma - \frac{\Sigma a  a^\top\Sigma}{s^2}\Bigr).
\]
Substituting $\mathrm{Var}[T \mid T \ge b_0] = s^2(1 - \lambda(\lambda-\alpha))$ from \eqref{eq:T-moments} and simplifying, the $\Sigma a a^\top\Sigma/s^2$ terms combine to leave \eqref{eq:cond-cov}.

\emph{Expected severity \eqref{eq:severity}.} On $\mathcal{X}_0$, $z(x) = C_0 x + d_0$, so $w^\top Z_0(x) - \gamma = w^\top C_0 x + w^\top d_0 - \gamma = a^\top x - b_0$, using $a = (w^\top C_0)^\top$ and $b_0 = \gamma - w^\top d_0$. Hence, by \eqref{eq:T-moments}, 
\[
    \mathbb{E}[w^\top Z_0(X) - \gamma \mid X \in S]
    = \mathbb{E}[T \mid T \ge b_0] - b_0
    = a^\top x_0 + s\lambda - b_0 = s(\lambda - \alpha),
\]
since $b_0 - a^\top x_0 = s\alpha$. Tail quantiles of the severity follow from those of the univariate truncated normal $T \mid T \ge b_0$. 

\emph{Sensitivity \eqref{eq:sensitivity}.} With $\nu_{\mathrm{local}} = \Phi(-d^*_\Sigma(x_0))$ and $d^*_\Sigma(x_0) = (b_0 - a^\top x_0)/s$, the chain rule gives $\partial d^*_\Sigma/\partial x_0 = -a/s$, so
\[
    \frac{\partial \nu_{\mathrm{local}}}{\partial x_0}
    = -\phi(-d^*_\Sigma(x_0)) \frac{\partial d^*_\Sigma}{\partial x_0}
    = \phi(d^*_\Sigma(x_0)) \frac{a}{s} = \phi(\alpha) \frac{a}{s},
\]
using the symmetry $\phi(-u) = \phi(u)$ and $\alpha = d^*_\Sigma(x_0)$. This is \eqref{eq:sensitivity}. \hfill \qed

\end{document}